\documentclass{llncs}

\usepackage{times}

\usepackage[T1]{fontenc}
\usepackage[algoruled,norelsize,vlined]{algorithm2e}
\usepackage{algorithmic}%
\usepackage[table]{xcolor}
\usepackage{hhline,arydshln}
\usepackage{rotating}
\usepackage{multirow}
\usepackage{subfig}
\usepackage{url}
\usepackage{hyperref}
\usepackage{pict2e}
\usepackage{amsmath, amssymb, amsfonts}
\usepackage{multirow,hhline,arydshln}
\usepackage{color}
\usepackage{pgf}
\usepackage{tikz}
\usepackage{graphicx}
\usepackage{latexsym}
\usepackage{xspace}
\usetikzlibrary{arrows,automata}
\usepackage{epstopdf}
\usepackage{subfig}

\def\N{\mbox{$\mathbb{N}$}}

\def\inf{\mbox{$\preceq$}}

\def\vide{\mbox{$\Box$}} 
 
\def\card{\#}

\def\gap{\mbox{$gap[M,N]$}\xspace}
\def\MNgap{[M, N]\xspace}
\def\ginf{\inf^{\MNgap}\xspace}
\def\gapCover{cover_{\SDB}^{\MNgap}\xspace}
\def\gapSup{sup_{\SDB}^{\MNgap}\xspace}

\newcommand{\gapMN}[2]{\mbox{$gap[#1,#2]$}}
\newcommand{\gapinf}[2]{\ensuremath{#1\ginf #2}\xspace}
\newcommand{\gapMNinf}[3]{\ensuremath{#1\inf^{#3} #2}\xspace}

\newcommand{\RE}[2]{\ensuremath{Ext^{[M,N]}_R(#1,#2)}\xspace}

\newcommand{\REG}[3]{\ensuremath{Ext^{#3}_R(#1,#2)}\xspace}
\newcommand{\RF}[2]{\ensuremath{\mathcal{RF}^{[M,N]}(#1,#2)}\xspace}
\newcommand{\cspadeSG}[2]{\ensuremath{\tt cSpade[#1,#2]}\xspace}

\newcommand{\angx}[1]{{{\mbox{$\langle #1 \rangle$}}}}

\newcommand{\SDB}{SDB\xspace}

\newcommand{\I}{\mathcal{I}} 	
 	
\newcommand{\Li}{\mathcal{L_I}}

\usepackage{etoolbox}
\newbool{seeAll}
\usepackage[normalem]{ulem}

\def\cspade{\texttt{cSpade}\xspace}

\def\cpsm{\texttt{global-p.f}\xspace}
\def\cps{\texttt{decomposed-p.f}\xspace}
\def\gapcp{\texttt{GAP-SEQ}\xspace}

\def\pp{\textsc{Prefix-Projection}\xspace}

\def\gapCP{\textsc{GAP-SEQ}\xspace}
\newcommand{\gapPP}[2]{\ensuremath{\gapcp[#1,#2]}\xspace}

\begin{document}
\title{A global Constraint for mining Sequential Patterns with GAP constraint}

\author{Amina Kemmar$^{1}$ \and Samir
  Loudni$^2$ \and Yahia Lebbah$^{1}$\\ Patrice Boizumault$^2$  \and Thierry
  Charnois$^3$} 
\institute{
$^1$LITIO -- University of Oran 1 -- Algeria, EPSECG of Oran -- Algeria\\ 
$^2$ GREYC (CNRS UMR 6072) -- University of Caen -- France\\ 
$^3$LIPN (CNRS UMR 7030) -- University PARIS 13 -- France \\
}

\maketitle

\begin{abstract}
Sequential pattern mining (SPM) under gap constraint is a challenging task.
Many efficient specialized methods have been developed but they are all suffering from a lack of genericity.
The Constraint Programming (CP) approaches are not so effective because of the size of their encodings.
In~\cite{DBLP:conf/cp/KemmarLLBC15}, we have proposed the global constraint \pp for SPM which remedies to this drawback.
However, this global constraint cannot be directly extended to support gap constraint.
In this paper, we propose the global constraint \gapCP enabling to handle SPM with or without gap constraint.
\gapCP relies on the principle of right pattern extensions.
Experiments show that our approach clearly outperforms both CP approaches and the state-of-the-art \cspade method on large datasets.
\end{abstract}
 
\section{Introduction}
\label{section:introduction}
Mining sequential patterns (SPM) is an important task in data mining. 
There are many useful applications, including discovering
changes in customer behaviors, detecting intrusion from
web logs and finding relevant genes from DNA sequences. 
In recent years many studies have focused on SPM with gap
constraints~\cite{DBLP:conf/cikm/Zaki00,DBLP:journals/tkdd/ZhangKCY07}. 
Limited gaps allow a mining process to bear a certain degree of
flexibility among correlated pattern elements in the original sequences. 
For example, \cite{DBLP:conf/icdm/JiBD05} analyses purchase behaviors to reflect products usually bought by customers at regular time intervals according to time gaps. 
In computational biology, the gap constraint helps discover periodic
 patterns with significant biological and 
 medical values~\cite{DBLP:journals/cbm/WuZHA13}. 

Mining sequential patterns under gap constraint (GSPM) is a challenging
task, since the {\it apriori property} does not hold for this problem:
{\it a subsequence of a frequent sequence is not necessarily frequent}. 
Several specialized approaches have been
proposed~\cite{DBLP:conf/icdm/JiBD05,Li2012,DBLP:journals/tkdd/ZhangKCY07}  
but they have a lack of genericity to handle simultaneously various types of constraints. 
Recently, a few proposals~\cite{DBLP:conf/ecai/CoqueryJSS12,DBLP:conf/ictai/KemmarULCLBC14,metivierLML13,NegrevergneCPIAOR15}
have investigated relationships between GSPM and constraint programming (CP) 
in order to provide a declarative approach, while exploiting efficient and generic solving methods. 
%
But, due to the size of the proposed encodings, these CP methods are not as efficient as specialized ones. 
%
More recently, we have proposed the global constraint \pp for SPM which remedies to this drawback~\cite{DBLP:conf/cp/KemmarLLBC15}. 
However, as this global constraint uses the projected databases principle, it cannot be directly extended to support gap constraint. 

In this paper, we introduce the global constraint \gapCP enabling to handle SPM with or without gap constraint.
%
\gapCP relies on the principle of right pattern extension and its filtering exploits the prefix anti-monotonicity property of the gap constraint
to provide an efficient pruning of the search space.
\gapCP enables to handle simultaneously different types of constraints 
and its encoding does not require any reified constraints nor any extra variables.
Finally, experiments show that our approach clearly outperforms CP
approaches as well as specialized methods for GSPM and achieves scalability while it is a major issue for CP approaches.

The paper is organized as follows.
Section~\ref{section:preliminaries}
introduces the prefix anti-monotonicity of the gap constraint as well as right pattern extensions
that will enable an efficient filtering.
Section~\ref{section:related-works} provides a critical review of specialized
methods and CP approaches for sequential pattern mining under gap constraint. 
Section~\ref{sec:gap} presents the global constraint \gapCP.
Section~\ref{section:experimentations} reports experiments we
performed. Finally, we conclude and draw some perspectives.

\vspace*{-.2cm}
\section{Preliminaries}
\label{section:preliminaries}
\setlength{\textfloatsep}{0pt}
\newcommand{\yl}[2][Yahia]
{\fboxsep=1pt\fbox{\tiny\color{blue} #1}{{\color{red} #2}}}

\vspace*{-.2cm} 
First, we provide the basic definitions for GSPM.
Then, we show that the anti-monotonicity property of frequency of SPM does not hold for GSPM.
Finally, we introduce right pattern extensions that will enable an efficient filtering for GSPM.

\vspace*{-.2cm}
\subsection{Definitions}
\label{sec-2.1}
Let $\I$ be a finite set of distinct {\em items}. 
The language of sequences  corresponds to $\Li=\I^n$ where $n\in\N^+$. 

\begin{definition}[sequence, sequence database]
A sequence $s$ over $\Li$ is an ordered list
$\langle s_1 s_2 \ldots s_n \rangle$, where $s_i$, $1 \leq i \leq n$,
is an item. $n$ is called the length of the sequence $s$.
A sequence database $\SDB$ is a set of tuples $(sid, s)$,
where $sid$ is a sequence identifier and $s$ a sequence denoted by $\SDB[sid]$. 
\end{definition}

\vspace*{-.2cm}
We now define the subsequence relation $\ginf$ under \gap constraint
which restricts the allowed 
distance between items of subsequences in sequences. 

\begin{definition}[subsequence relation $\ginf$ under \gap]
$\alpha = \langle \alpha_1 \ldots \alpha_m \rangle$ is a 
  subsequence of $s = \langle s_1 \ldots s_n \rangle$, under \gap,
denoted by $(\gapinf{\alpha}{s})$, if $m \leq n$ and, for all $1 \leq
i \leq m$, there exist integers  
$1 \leq j_1 \leq \ldots \leq j_m\leq n$, such that 
$\alpha_i=s_{j_i}$, and $\forall k\in\{1, ..., m-1\}, M\leq j_{k+1} -
j_{k} - 1 \leq N$.  
In this context, the pair $(s,[j_1,j_m])$ denotes an {\bf occurrence}
of $\alpha$ in $s$, where $j_1$ and $j_m$ represent the positions of
the first and last items of $\alpha$ in $s$.
We say that $\alpha$ is contained in $s$ or $s$ is a super-sequence of 
$\alpha$  under \gap. We also say that $\alpha$ is a \gap~{\bf
  constrained pattern} in $s$. 
\begin{itemize}
\vspace*{-.2cm}
\item 
Let $AllOcc(\alpha, s) = \{[j_1,j_m] \,|\, (s,[j_1,j_m])$ is an occurrence of $\alpha$ in $s\}$
be the set of all the occurrences of some sequence $\alpha$ under \gap{} in $s$.
\item  
Let $AllOcc(\alpha,\SDB) =\{(sid,AllOcc(\alpha,\SDB[sid])) \,|\, (sid, \SDB[sid]) \in \SDB \} $
be the set of all the occurrences of some sequence $\alpha$ under
  \gap{} in \SDB. 
\item 
Let $gap[M, \infty]$ and $gap[0, N]$ the {\bf minimum and the maximum gap} constraints respectively.  
\item 
The relation $\inf$ stands for $\inf^{[0,\infty]}$ where the gap  constraint is inactive. 
\end{itemize}
\end{definition}

For example, the sequence $\angx{BABC}$ is a super-sequence of
$\angx{AC}$ under \gapMN{0}{2}: $\gapMNinf{\angx{AC}}{\angx{BABC}}{[0,
  2]}$.

\begin{definition}[prefix, postfix]  
Let $\beta=\angx{\beta_1 \ldots \beta_n}$ be a sequence. The sequence $\alpha = \angx{\alpha_1
  \ldots \alpha_m}$ where $m \leq n$ is called the prefix of $\beta$ iff $\forall i \in
[1..m], \alpha_i=\beta_i$. The sequence $\gamma =\angx{\beta_{m+1}
  \ldots \beta_n}$ is called the postfix of $s$ w.r.t. $\alpha$. With
the standard concatenation operator "$concat$", 
we have $\beta = concat(\alpha, \gamma)$.  
\end{definition}

The cover of a sequence $\alpha$ in $\SDB$ is the set of all tuples in
$\SDB$ in which $\alpha$ is contained. 
The support of a sequence $\alpha$ in $\SDB$  is the cardinal of its
cover. 

\begin{definition}[coverage and support under \gap]
Let $\SDB$ be a sequence database and $\alpha$ a sequence. 
$\gapCover(\alpha)$$=$$\{(sid, s) \in \SDB \, | \, \alpha \,\inf^{\MNgap}\, s\}$ and
\\$\gapSup(\alpha) = \card \gapCover(\alpha)$. 
\end{definition}

\begin{definition}[\gap constrained sequential pattern mining (GSPM)]
Given a sequence database $\SDB$, a minimum support threshold
$minsup$ and a gap constraint \gap. The problem of \gap
constrained sequential pattern mining is to find all subsequences
$\alpha$ such that $\gapSup(\alpha) \geq minsup$. 

\end{definition}

\begin{table}[t]
	\begin{center}
		\scalebox{.85}{
			\begin{tabular}{|l|l|}	
				\hline
                                ~sid~ &~Sequence~ \\
				\hline
				 $1$ & $\angx{ABCDB}$  \\
				 $2$ & $ \angx{ACCBACB}$ \\
				 $3$ & $ \angx{ADCBEEC} $ \\ 
				 $4$  & $\angx{AACC}$ \\
				\hline
			\end{tabular} 
		}  
                        \caption{A sequence database example $\SDB_1$.}
		\label{tab:SDB}
	\end{center}
\end{table}	

\vspace*{-.4cm} 
\begin{example}
\label{example:1}
Table~\ref{tab:SDB} represents
  a sequence database of four sequences 
where the set of items is $ \I = \{A, B, C, D, E\}$. Let 
the sequence $\alpha= \angx{A C}$. The occurrences under \gapMN{0}{1}
of $\alpha$ in $\SDB_1[2]$ is given by   
$AllOcc(\alpha,\SDB_1[2]) = \{[1,2]), [1,3], [5,6]\}$. 
We have $cover_{\SDB_{1}}^{[0,1]} (\alpha)=\{(1,
s_1), (2, s_2), (3, s_3),(4, s_4)\}$. If we consider $minsup=2$,
$\alpha$ is a \gapMN{0}{1} constrained sequential pattern because
$sup_{\SDB_1}^{[0,1]} (\alpha) \ge 2$.   
\end{example}


\subsection{Prefix anti-monotonicity of \gap}
\label{sec-2.1-pb}
Most of
SPM algorithms rely on the {\it anti-monotonicity property of frequency}
\cite{DBLP:conf/icde/AgrawalS95} to reduce the search space: all the
subsequences of a frequent sequence are frequent as well (or,
equivalently, if a subsequence is infrequent, then no super-sequence of
it can be frequent). However, this property does not hold for the gap
constraint, and more precisely for the maximum gap constraint.   
A simple illustration from our running example suffices to show that
sequence $\angx{A B}$ is not a sequential pattern under \gapMN{0}{1}
(for  $minsup = 3$)  whereas sequence $\angx{A C B}$ is a
\gapMN{0}{1} constrained sequential pattern. 
As a consequence, one needs to use other techniques for pruning the search
space. 
The following proposition
shows how the
{\it prefix anti-monotonicity property} can be exploited to ensure the 
anti-monotonicity of the gap constraint. 

\begin{definition}[prefix anti-monotone property]
A constraint $c$ is called prefix anti-monotone
if for every sequence $\alpha$ satisfying $c$, every prefix
 of $\alpha$ also satisfies the constraint.   
\end{definition}

\begin{proposition}
\label{prefix-gap-property}
\gap is prefix anti-monotone.
\end{proposition}

\begin{proof}
Let $\alpha=\angx{\alpha_1\dots \alpha_m}$ and $s=\angx{s_1 \dots s_n}$ be two
sequences s.t. $\gapinf{\alpha}{s}$ and $m \leq n$. By definition, there exist
integers $1 \leq j_1 \leq \ldots \leq j_m\leq n$, such that
$\alpha_i=s_{j_i}$, and $\forall k\in\{1, ..., m-1\}, M\leq j_{k+1} -
j_{k} - 1 \leq N$. As a consequence, the property also holds for
every prefix of $\alpha$.
$\Box$ 
\end{proof}

Hence, if a sequence $\alpha$ does not satisfy \gap, then all
sequences that have $\alpha$ as prefix will not satisfy this
constraint. Sect.~\ref{consistency} shows how this property
can be exploited to provide an efficient filtering.

\vspace*{-.2cm} 
\subsection{Right pattern extensions}
\label{right:pattern:extension}
Right pattern
extensions of some pattern $p$ gives all the possible subsequences which
can be appended at right of $p$ to form a \gap constrained pattern. 
According to proposition ~\ref{prefix-gap-property}, 
the set of all items  locally frequent within the right pattern
extensions of $p$ in \SDB can be used  to extend $p$.
In the following, we introduce an operator allowing to compute all the
right pattern extensions of a pattern w.r.t. \gap. 

\begin{definition}[Right pattern extensions]
\label{def:left:right}
Given some sequence $(sid, s)$ and a pattern $p$ 
s.t. $\gapinf{p}{s}$. The {\bf right pattern extensions} of $p$ in 
$s$, denoted by $\RE{p}{s}$, is the collection of legal subsequences of $s$ 
located at the right of $p$ and satisfying \gap. 
To define $\RE{p}{s}$, we need to define $BE^{[M, N]}(p, s)$ {\bf basic right extensions} :
\begin{footnotesize}
\[BE^{[M, N]}(p, s) = \bigcup_{[j_1,j_m]\in AllOcc(p, s)}
\{(j_m, \mathtt{SubSeq}(s,j_m+M+1,
min(j_m+N+1,\#s)))\}\] 

\end{footnotesize}
\[\mbox{ where }\mathtt{SubSeq}(s, i_1, i_2) = \left\{
\begin{array}{ll}
\angx{s[i_1], ..., s[i_2]} & \mbox{ if } i_1\leq i_2\leq \#s\\
\angx{} & \mbox{ otherwise }
\end{array}\right.
\] 
 
\noindent
Right pattern extensions $\RE{p}{s}$ is defined as follows:  




\begin{footnotesize}
\begin{equation}\label{def-ExtR}
\RE{p}{s} = \left\{
\begin{array}{lr}
 \{Sb\, | \,(j'_{m}, Sb) \in BE^{[M,N]}(p, s) \wedge & \mbox{ if } N \geq \#s \\ 
 \hspace*{0.8cm} j'_{m} = \min_{(j_m, Sb)\in BE^{[M, N]}(p,s)}\{j_m\}\} & \\
\bigcup_{(j_m, Sb)\in BE^{[M, N]}(p, s)}\{Sb\} & \mbox{ otherwise } 
\end{array}\right.
\end{equation}
\end{footnotesize}




\end{definition}
\noindent

Formula (\ref{def-ExtR}) states exactly the set of all possible
extensions of pattern $p$ within $s$. In case where $(N \geq \#s)$,
since that any extension from $BE^{[M,N]}(p, s)$ always reaches the
end of the sequence $s$, thus all possible extensions can be
aggregated within one unique extension going from the lowest starting
position $j'_{m} = \min_{(j_m, Sb)\in BE^{[M, N]}(p,s)}\{j_m\}$. We
point out that these cases $(N \geq \#s)$  cover the special case of no
gap $gap[0, \infty]$. 

The right pattern extensions of $p$ in \SDB is the collection of all its right 
pattern extensions in all sequences of \SDB:  
\begin{small}
\begin{equation}\label{def-ExtR-SDB}
\RE{p}{\SDB}=\bigcup_{(sid, s) \in \SDB} \{(sid, \RE{p}{s})\}
\end{equation}
\end{small}


\vspace*{-.3cm} 
\begin{example}\label{example:LR}
Let $p_1=\angx{AC}$ be a pattern and the gap constraint be
\gapMN{0}{1}. We have
$AllOcc(p_1,\SDB_1[2]) = \{[1,2]), [1,3], [5,6]\}$.  
The right pattern extensions of $p_1$ in $SDB_1[2]$ is equal to
$\REG{p_1}{SDB_1[2]} {[0,1]} = 
\{\angx{CB},\angx{BA}, \angx{B}\}$. The right pattern extensions of
$p_1$ in $\SDB_1$ is given by     
$\REG{p_1}{SDB_1} {[0,1]} =\{(1,\{\angx{DB}\}), (2,\{\angx{CB},$ $\angx{BA}, \angx{B}\}),
(3,\{\angx{BE}\}), (4,\{\angx{C}\})\}$. 

Let the gap constraint be \gapMN{0}{\infty}. To compute
$\REG{p_1}{SDB_1[2]} {[0,\infty]}$,  
only the first occurrence of $p_1$ in $\SDB_1[2]$ need to be
considered (i.e. $[1,2]$) (cf. Definition~\ref{def:left:right}). 
Thus, $\REG{p_1}{SDB_1[2]} {[0,\infty]}$
$=\{\angx{CBACB}\})$. The right pattern extensions of $p_1$ in
$\SDB_1$ is equal to $\REG{p_1}{SDB_1} {[0,\infty]}$
$=\{(1,\{\angx{DB}\})$, $(2,\{\angx{CBACB}\})$, $(3,\{\angx{
  BEEC}\})$, $(4,\{\angx{C }\})\}$.    
\end{example}

Given a \gap constrained pattern $p$ in \SDB, according to proposition~\ref{prefix-gap-property}, 
the set of all items  locally frequent within the right pattern extensions of $p$ in \SDB can be used 
to extend $p$. Proposition~\ref{prop:SupCount} establishes the support count 
of a sequence $\gamma$ w.r.t. its right pattern extensions.

\vspace*{-.1cm} 
\begin{proposition}[Support count]\label{prop:SupCount}
For any sequence $\gamma$ in $\SDB$ with prefix $\alpha$ and postfix $\beta$
s.t. $\gamma= concat(\alpha,\beta$), 
$sup_{\SDB}^{\MNgap}(\gamma) = sup_{\RE{\alpha}{\SDB}}(\beta)$. 
\end{proposition}

This proposition ensures that only the sequences in $\SDB$ grown from
$\alpha$ need to be considered for the support count of a sequence
$\gamma$. From proposition \ref{prop:SupCount}, we can derive the following 
proposition to establish a condition to check when a pattern is a \gap 
constrained sequential pattern. 

\vspace*{-.1cm} 
\begin{proposition}
\label{prop:consistency}
Let \SDB be a sequence database and a minimum support threshold $minsup$. 
A pattern $p$ is a \gap  constrained sequential pattern in \SDB if and
only if the following condition holds:  
$\#\RE{p}{\SDB} \geq minsup$
\end{proposition}

\vspace*{-.35cm} 
\begin{example}
Let $minsup$ be $2$ and the gap constraint be \gapMN{0}{1}. 
From Example~\ref{example:LR}, we have $\#\REG{p_1}{\SDB_1}{[0,1]} = 4 \geq minsup$.  
Thus, $p_1 = \angx{AC}$ is a \gapMN{0}{1} constrained sequential pattern. 
The  locally frequent items within the right pattern extensions 
$\REG{p_1}{\SDB_1}{[0,1]}$ of $p_1$  are
$B$ and $C$ with supports of $3$ and $2$ respectively. 
According to proposition~\ref{prop:SupCount}, $p_1$ can be extended to  
two \gapMN{0}{1} constrained sequential patterns $\angx{ACB}$ and $\angx{ACC}$ . 
\end{example}

\vspace*{-.2cm}
\section{Related works}
\label{section:related-works}
\vspace*{-.15cm}

\paragraph{\bf Specialized methods for GSPM.} 
The SPM was first proposed in~\cite{DBLP:conf/icde/AgrawalS95}. 
Since then, many efficient specialized approaches have been proposed~\cite{Ayres:2002,DBLP:conf/icde/PeiHPCDH01,DBLP:journals/ml/Zaki01}. 
There are also several methods focusing on gap constraints.  
Zaki~\cite{DBLP:conf/cikm/Zaki00} first proposed \cspade, a depth-first search
based on a vertical database format, incorporating max-gap, max-span and length constraints.  
Ji and al.~\cite{DBLP:conf/icdm/JiBD05} and Li and al.~\cite{ChunLi2008}
studied the problem of mining frequent patterns with gap constraints. 
In \cite{DBLP:conf/icdm/JiBD05}, a minimal distinguishing subsequence that occurs frequently in the positive sequences and
infrequently in the negative sequences is proposed, where the maximum gap constraint is defined. 
In \cite{ChunLi2008}, closed frequent patterns with gap constraints are mined. 
All these proposals, though efficient, lack of genericity to handle simultaneously various types of constraints.

\paragraph{\bf CP Methods for GSPM.}
There are few methods for 
SPM with gap constraints using CP. 
\cite{metivierLML13} have proposed to model
a sequence 
using an automaton capturing all subsequences that can occur in it. 
The gap constraint is encoded by removing from the automaton all transitions 
that does not respect the gap constraint.
\cite{DBLP:conf/ictai/KemmarULCLBC14} have proposed a CSP model for
SPM with explicit wildcards\footnote{A wildcard is a special symbol that
matches any item of $\I$ including itself.}. The gap constraints is
enforced using the regular global constraint. 
\cite{NegrevergneCPIAOR15} have 
proposed two CP encodings for the SPM. The first one uses a global
constraint to encode the subsequence relation (denoted \cpsm), 
while the second one (denoted \cps) encodes explicitly this relation
using additional variables and constraints in order to support
constraints like gap. However, all these proposals usually lead to
constraints network of huge size. Space complexity is clearly
identified as the main bottleneck behind the competitiveness of these
declarative approaches. In~\cite{DBLP:conf/cp/KemmarLLBC15}, we have
proposed the global constraint \pp for sequential pattern mining which remedies
to this drawback. However, this constraint
cannot be directly
extended to handle gap constraints. This requires changing 
the way the subsequence relation is encoded.

The next section introduces
the global constraint \gapCP enabling to handle SPM with or without gap constraints.
\gapCP relies on the prefix anti-monotonicity of the gap constraint and on the right pattern extensions to provide an efficient filtering.
This global constraint does not require any reified constraints nor any extra variables to encode the subsequence relation.


\vspace*{-.3cm}
\section{\gapCP global constraint}
\label{sec:gap}
\newcommand\codex[1]{\mbox{\sc #1}}
\vspace*{-.15cm}
This section is devoted to the \gapCP global constraint.
Section~\ref{sec:pattern} defines the \gapCP global constraint and presents the CSP modeling. 
Section~\ref{consistency} shows how the filtering can take advantage of
the prefix anti-monotonicity property of the \gap constraint (see Proposition~\ref{prop-filtering})
and of the right pattern extensions (see Proposition~\ref{prop-consistency}) 
to remove inconsistent values from the domain of a future variable.
Section~\ref{Filtering} details the filtering algorithm and
Section~\ref{complexities} provides its temporal and spatial complexities.

\vspace*{-.25cm}
\subsection{CSP modeling for GSPM}
\label{sec:pattern}

\noindent
A {\it Constraint Satisfaction Problem} (CSP) consists of 
a set $X$ of $n$ variables, 
a domain $\mathcal{D}$ mapping each variable $X_i \in X$ to a finite set of values $D(X_i)$, 
and a set of constraints $\mathcal{C}$. 
An assignment $\sigma$ is a mapping from variables in $X$ to
values in their domains.
%
A constraint $c \in \mathcal{C}$ is a subset of the 
cartesian product of the domains of the variables that occur in $c$. 
The goal is to find an assignment such that all constraints are satisfied.  

\smallskip
\noindent {\bf (a) Variables and domains.} 
Let $P$ be the unknown pattern of size $\ell$ we are looking for. 
The symbol $\vide$ stands for an empty item and denotes the end of a
sequence.  
We encode the unknown pattern $P$ of maximum length $\ell$ with a
sequence of $\ell$ variables $\angx{P_1,P_2,\ldots,P_\ell}$. Each
variable $P_j$ represents the item in the $jth$ position of the
sequence. The size $\ell$ of $P$ is determined by the length of the
longest sequence of $\SDB$. The domains of variables are defined as
follows: (i) $D(P_ 1)= \I$ to avoid the empty 
sequence, and (ii) $\forall i \in [2\ldots\ell],  D(P_ i)=
\I\cup\{\vide\}$. To allow patterns with less than $\ell$ items, we
impose that $\forall 
i \in [2.. (\ell$$-$$1)], (P_i=\vide) \rightarrow (P_{i+1} = \vide)$.

\smallskip
\noindent {\bf (b) Definition of \gapCP.}
The global constraint \gapCP encodes both subsequence relation $\ginf$
under gap constraint \gap and minimum frequency constraint directly on the data. 

\begin{definition}[\gapCP global constraint]
Let $P = \angx{P_1,P_2,\ldots,P_\ell}$ be a pattern of size $\ell$ and
\gap be the gap constraint.  
$\angx{d_1, ..., d_{\ell}} \in D(P_1)\times \ldots \times D(P_\ell)$
is a solution of\mbox{ }$\gapCP(P,\SDB,minsup,M,N)$ iff
$\gapSup(\angx{d_1, ..., d_{\ell}}) \geq minsup$. 
\end{definition}

\begin{proposition}
\label{prop-solution}
$\gapCP (P, \SDB,minsup,M,N)$ has a solution iff  
there exists an assignment $\sigma = \angx{d_1, ..., d_{\ell}}$ of
variables of $P$ s.t. $\#\RE{\sigma}{\SDB} \geq minsup$.
\end{proposition}
\vspace*{-.1cm}
\noindent
{\it Proof: }
This is a direct consequence of proposition \ref{prop:consistency}. 
$\Box$ 

\smallskip
\noindent {\bf (c) Other SPM constraints} can be directly modeled as follows:

\noindent
- \emph{Minimum Size} constraint restricts the number of items of a pattern  to be at least $\ell_{min}$:
$minSize(P, \ell_{min})  \equiv
\bigwedge_{i=1}^{i=\ell_{min}} (P_i \neq \square)$

\noindent
- \emph{Maximum Size} constraint restricts the number of items of a pattern to be at most $\ell_{max}$:  
$maxSize(P, \ell_{max})  \equiv \bigwedge_{i=\ell_{max}+1}^{i=\ell} (P_i = \square)$

\noindent
- \emph{Membership} constraint states that a subset of items $V$ must belong (or not) to the extracted patterns. 
$item(P, V) \equiv \bigwedge_{t \in V} \mbox{\tt Among}(P,\{t\},l,u)$
enforces that items of $V$ should occur at least $l$ times and at most $u$ times in $P$.  
To forbid items of $V$ to occur in $P$, $l$ and $u$ must be set to $0$.

\subsection{Principles of filtering}
\label{consistency}

\noindent {\bf (a) Maintaining a local consistency.}
SPM is a challenging task due to the
exponential number of candidates that should be parsed to find the
frequent patterns. For instance, with $k$ items there are $O(n^k)$
potentially candidate patterns of length at most $k$ in a sequence of
size $n$. With gap constraints, the
problem is even much harder since the complexity of checking for
subsequences taking a gap constraint into account is higher than
the complexity of the standard subsequence
relation. Furthermore, the NP-hardness of
mining maximal\footnote{A sequential pattern $p$ is
maximal if there is no sequential pattern $q$ such that $p \inf q$.}
frequent sequences was established in \cite{Guizhen-2006} by proving 
the \#P-completeness of the problem of counting the number of maximal
frequent sequences. Hence, ensuring {\it Domain Consistency} (DC) for \gapCP
i.e., finding, for every variable $P_j$, a value $d_j \in D(P_j)$, 
satisfying the constraint is NP-hard.  

So, the filtering of \gapCP constraint maintains a consistency lower than DC.
This consistency is based on specific properties of the \gap constraint
and resembles forward-checking (regarding Proposition~\ref{prop-consistency}). 
\gapCP is considered as a global constraint, since all variables
share the same internal data structures that awake and drive the
filtering. 
The prefix anti-monotonicity property of the \gap constraint (see Proposition~\ref{prop-filtering})
and of the right pattern extensions (see Proposition~\ref{prop-consistency})
will enable to remove inconsistent values from the domain of a future variable.

\smallskip
\noindent {\bf (b) Detecting inconsistent values.}
Let $\RF{\sigma}{\SDB}$ be the set of locally frequent
items within the right pattern extensions, defined by $\{v \in \I
\,|\, \#\{sid \,|\, (sid, E) \in $ $\RE{\sigma}{\SDB} \wedge (\exists
\alpha \in E \wedge \angx{v}\inf \alpha)\} \geq minsup\}$. The following proposition 
characterizes values, of a future (unassigned) variable $P_{j+1}$,
that are consistent with the current assignment of variables $\angx{P_1, \dots, P_j}$. 

\vspace*{-.1cm}
\begin{proposition}
\label{prop-consistency}
Let~\footnote{We indifferently denote $\sigma$ by $\angx{d_1,
    \dots, d_j}$ or by $\angx{\sigma(P_1), \dots, \sigma(P_{j})}$.} $\sigma$
$=\angx{d_1, \dots, d_j}$ be a current assignment of
variables $\angx{P_1, \dots, P_j}$, $P_{j+1}$ be a future variable. 
A value $d \in D(P_{j+1})$ occurs in a solution for the global
constraint $\gapCP(P, \SDB,
minsup,M,N)$ iff $d \in \RF{\sigma}{\SDB}$. 
\end{proposition}
\vspace*{-.1cm}
\noindent
{\it Proof: } 
Assume that $\sigma =\angx{d_1, \dots, d_j}$ is \gap constrained
sequential pattern in $\SDB$. Suppose that value
$d \in D(P_{j+1})$ appears in $\RF{\sigma}{\SDB}$. As the local
support of $d$ within the right extensions is equal to
$sup_{\RE{\sigma}{\SDB}}(\angx{d})$, from
proposition~\ref{prop:SupCount} we have $\gapSup(concat(\sigma, 
\angx{d})) = sup_{\RE{\sigma}{\SDB}}(\angx{d})$. Hence, we can get a
new assignment $\sigma \cup \angx{d}$ that satisfies the constraint. 
Therefore, $d \in D(P_{j+1})$ participates in a solution.
$\Box$

From proposition~\ref{prop-consistency} and according to the
prefix anti-monotonicity property of the gap constraint, we can derive
the following pruning rule: 

\begin{proposition}
\label{prop-filtering}
Let $\sigma=\angx{d_1, \dots, d_j}$ be a current assignment of
variables $\langle P_1,$ $\dots,$ $P_j \rangle$. All values $d \in
D(P_{j+1})$ that are not in $\RF{\sigma}{\SDB}$ can be removed
from the domain of variable $P_{j+1}$.  
\end{proposition}


\begin{example}
Consider the running example of Table~\ref{tab:SDB}, let 
$minsup$ be $2$ and the gap constraint be $gap[1,2]$. 
Let $P = \angx{P_1,P_2,P_3,P_4}$ with $D(P_1) = \I$ and 
$D(P_2) = D(P_3) = D(P_4) = \I\cup\{\vide\}$. 
Suppose that $\sigma(P_1) = A$. We have
$\REG{\angx{A}}{\SDB_1}{[1,2]}$ $=$ $\{(1,\{\angx{CD}\}),
(2,\{\angx{CB},\angx{B}\}), (3,\{\angx{CB}\}), 
(4,\{\angx{CC},\angx{C}\})\}$. As $B$ and $C$ are the only locally
frequent items in $\REG{\angx{A}}{\SDB_1}{[1,2]}$, $\gapCP$ will
remove values $A$, $D$ and $E$ from $D(P_2)$.  
\end{example} 

The filtering of \gapCP relies on Proposition~\ref{prop-filtering} and is detailed in the next section.

\subsection{Filtering algorithm}
\label{Filtering}

\begin{algorithm}[t]
\begin{scriptsize}
\SetKwFunction{getRExt}{\codex{getRightExt}}
\SetKwFunction{getFItems}{\codex{getFreqItems}}
\caption{\small \codex{FILTER-GAP-SEQ}($\SDB$, $\sigma$, $j$, $P$,
  $minsup$, $M$, $N$) \label{algo:filter}} 
\KwData{$\SDB$: initial database; $\sigma$: current assignment
  $\angx{\sigma(P_1), \ldots,\sigma(P_j)}$; $minsup$: the minimum
  support threshold; $\mathcal{ALLOCC}$: internal data structure for
  storing occurrences of patterns in $\SDB$; $Ext_R$:
  internal data structure for storing right pattern extensions of
  $\sigma$ in $\SDB$.} 
\Begin
{
	\lnl{for0}$Ext_R \leftarrow \codex{getRightExt}(\SDB, \mathcal{ALLOCC}_{j-1},
  \sigma, M, N)$ \; 
	\lnl{for1}\If{$(\card{Ext_R} < minsup)$}{
    \lnl{for2} \Return False \;
	}            
    \lnl{pre1}\If{$(j\geq 2 \wedge \sigma(P_j) = \vide)$}{ 
    \lnl{pre2}\For{$k \leftarrow j+1$ \KwTo $\ell$}{
                                  \lnl{pre3}$P_{k}  \leftarrow \vide$ \;
    				}
   	}
	\Else{
    		\lnl{filter0} $\mathcal{RF} \leftarrow
                \getFItems(\SDB, Ext_R, minsup)$ \;
        \lnl{filter1} \ForEach{$a \in D(P_{j+1}) \, s.t. (a \neq \vide \wedge a \notin \mathcal{RF} )$}{
        \lnl{filter2}$D(P_{j+1}) \leftarrow  D(P_{j+1}) - \{a\}$ \;
                                 }                                  
        }
                        \lnl{back7} \Return True \;
     
}

\end{scriptsize}
\end{algorithm}

\vspace*{-.1cm}

\begin{algorithm}[t!]
\begin{scriptsize}
\SetKwFunction{ProcGetAllOcc}{\codex{Function getAllOcc}}
\SetKwFunction{getAllOcc}{\codex{getAllOcc}}
\caption{\small \codex{getRightExt}($\SDB$, $\mathcal{ALLOCC}_{j-1}$,
  $\sigma$, $M$, $N$) \label{algo:FB}}
\KwData{$\SDB$: initial database; $\mathcal{ALLOCC}_{j-1}$:
  occurrences of the partial assignment $\angx{\sigma(P_1), \dots,\sigma(P_{j-1})}$ in \SDB; 
  $\sigma$: the current partial assignment $\angx{\sigma(P_1),
    \ldots,\sigma(P_j)}$; $OccSet$: the positions of the first and
  last items of $\angx{\sigma(P_1),\ldots,\sigma(P_{j-1})}$ in
  $\SDB[sid]$; $Sb$: the positions of the first and last items
  of the right pattern extensions of $\sigma$ in $\SDB[sid]$.} 
\Begin
{
{
\lnl{p00}\If{$(\sigma=\angx{})$}{
\lnl{p01}\Return  $\{(sid, [1, \#s]) | (sid, s)\in \SDB\}$ \;
}
\lnl{p0}$\mathcal{ALLOCC}_{j} \leftarrow \getAllOcc(\SDB,
\mathcal{ALLOCC}_{j-1}, \sigma, M, N)$ \;
\lnl{p1}$Ext_R \leftarrow \emptyset$ \;
\lnl{p3}\ForEach{pair $(sid,OccSet) \in \mathcal{ALLOCC}_{j}$}{
\lnl{p4}$s \leftarrow \SDB[sid]$; $Sb \leftarrow \emptyset$  \;
\lnl{p5}\ForEach{pair $[j_1,j_m] \in OccSet$}{
\lnl{p6} $j_1' \leftarrow  j_m + M + 1$; $j_m'  \leftarrow min(j_m + N + 1,\#s)$ \;
\lnl{p7}\If{$(j_1' \leq j_m')$}{
\lnl{p9}$Sb \leftarrow Sb \cup \{(j_1',j_m')\}$ \;
}
}
\lnl{p11}$Ext_R \leftarrow Ext_R \cup \{(sid, Sb)\}$ \;
}
\lnl{p12}\Return $Ext_R$ \;
} 
}

\ProcGetAllOcc($\SDB$, $\mathcal{ALLOCC}_{j-1}$, $\sigma$, $M$, $N$) \;
\Begin
{
\lnl{pa1}$\mathcal{ALLOCC}_j \leftarrow \emptyset$; $inf \leftarrow
0$; $sup \leftarrow 0$\;
\lnl{pa2}\ForEach{pair $(sid,OccSet) \in \mathcal{ALLOCC}_{j-1}$}{
\lnl{pa3}$s \leftarrow \SDB[sid]$; $newOccSet \leftarrow \emptyset$; 
$redundant \leftarrow false$; $i \leftarrow 1$ \;
\lnl{pa4}\While{$(i \leq \#OccSet$ $\wedge \neg redundant)$}{
\lnl{pa4a}$[j_1,j_m] \leftarrow OccSet[i]$; $i \leftarrow i+1$\; 
\lnl{pa5}\If{$(\#\sigma = 1)$}{
\lnl{pa6}$inf \leftarrow 1$; $sup \leftarrow \#s$ \;
}
\Else{
\lnl{pa7}$inf \leftarrow j_m+M+1$; $sup \leftarrow min(j_m+N+1,\#s)$ \;
}
\lnl{pa8}$k \leftarrow inf$ \;
\lnl{pa9}\While{$((k \leq sup) \wedge (\neg redundant))$}{
\lnl{pa10}\If{$(s[k]=\sigma(P_j))$}{
\lnl{pa11}\If{$(\#\sigma = 1)$}{
\lnl{pa12}$newOccSet \leftarrow newOccSet \cup \{[k,k]\}$ \;
}
\Else{
\lnl{pa13}$newOccSet \leftarrow newOccSet \cup \{[j_1,k]\}$ \;
}
\lnl{pa14}\If{$(((sup=\#s) \wedge (\#\sigma > 1)) \vee (N\geq \#s))$}{
\lnl{pa15}$redundant \leftarrow true$ \;
}
}
\lnl{pa16}$k \leftarrow k+1$ \;
}
}
\lnl{pa17}\If{($newOccSet \neq \emptyset$)}{
\lnl{pa18}$\mathcal{ALLOCC}_j \leftarrow \mathcal{ALLOCC}_j \cup (sid, newOccSet)$ \;
}
}
\lnl{pa19}\Return $\mathcal{ALLOCC}_j$ \;
}
\end{scriptsize}
\end{algorithm}

\noindent
{\bf Algorithm~\ref{algo:filter}} describes the  pseudo-code of \gapCP filtering algorithm.
It takes as input:
the index $j$ of the last assigned variable in $P$,
the current partial assignment $\sigma=\langle\sigma(P_1),$ $\ldots,$ $\sigma(P_j)\rangle$, 
the minimum support threshold $minsup$, the minimum and the maximum gaps. 
The internal data-structure $\mathcal{ALLOCC}$ stores all the
intermediate occurrences of patterns in $\SDB$, 
where $\mathcal{ALLOCC}_j=AllOcc(\sigma, SDB)$, for $j\in[1\dots \ell]$.
If $\sigma = \angx{}$, then $\mathcal{ALLOCC}_0 = \{(sid,[1,\#s]) \,|\,$ $(sid,s)\in SDB\}$. 
$\mathcal{ALLOCC}_j$ is computed incrementally from
$\mathcal{ALLOCC}_{j-1}$ in order to enhance the efficiency.

Algorithm~\ref{algo:filter} starts by computing the right pattern
extensions $Ext_R$ of $\sigma$ in $\SDB$ (line~\ref{for0}) by calling 
function \codex{getRightExt} (see Algorithm~\ref{algo:FB}). Then, it  
checks whether the current assignment $\sigma$ satisfies the
constraint (see Proposition~\ref{prop-solution}) (line~\ref{for1}). 
If not, we stop growing $\sigma$ and return {\it False}. 
Otherwise, the algorithm checks if the last assigned
variable $P_j$ is instantiated to $\vide$ (line \ref{pre1}). 
If so, the end of the sequence is reached (since value $\vide$ can only
appear at the end) and the sequence $\angx{\sigma(P_1), \dots,
  \sigma(P_j)}$ is a $\gap$ constrained sequential pattern in
$\SDB$; hence, the algorithm sets the remaining $(\ell - j)$
unassigned variables to $\vide$ and returns {\it True} (lines
\ref{pre2}-\ref{pre3}). 
If $(P_j \neq \vide)$, the set of locally frequent
items, within the right pattern extensions $Ext_{R}$ of
$\sigma$ in $\SDB$, is computed by calling function
\codex{getFreqItems} (line~\ref{filter0}) and the domain of variable
$P_{j+1}$ is updated accordingly (lines~\ref{filter1}-\ref{filter2}).

\smallskip
\noindent
{\bf Algorithm \ref{algo:FB}} gives the pseudo-code of the function \codex{getRightExt}. 
First, if $\sigma$ is empty (i.e. $\#\sigma = 0$), all the sequences of $\SDB$ are considered as valid right pattern extensions;
the whole $\SDB$ should be returned.
Otherwise, the function \codex{getAllOcc} is called to compute the occurrences of $\sigma$ in $\SDB$ (line~\ref{p0}). 
Then, the algorithm processes all the entries of $\mathcal{ALLOCC}_j$, one by one (line \ref{p3}),
and, for each pair $(sid,OccSet)$, scans the occurrences of $\sigma$ in the sequence $sid$ (line~\ref{p5}).
For each occurrence $(j_1,j_m) \in OccSet$, the algorithm computes its right pattern extensions, 
i.e. the part of the sequence $sid$ which is in the range $[j_m+M+1, min(j_m+N+1,\#s)]$ (line~\ref{p6}).
If the new range is valid, it is added to the  set $Sb$ (line~\ref{p9}). 
After processing the whole entries in $OccSet$, the right pattern extensions of $\sigma$ 
in the sequence $sid$ are built and then added to the set $Ext_R$ (line~\ref{p11}). 
The process ends when all entries of $\mathcal{ALLOCC}_j$ have been considered. 
The right pattern extensions of $\sigma$ in $\SDB$ are then returned (line~\ref{p12}).

The function \codex{getAllOcc} computes incrementally
$\mathcal{ALLO}\-\mathcal{CC}_{j}$ from $\mathcal{ALLO}\-\mathcal{CC}_{j-1}$.
More precisely, lines~(\ref{pa5}-\ref{pa6}) and  
(\ref{pa11}-\ref{pa12}) are considered when the first variable $P_1$ 
is instanciated (i.e. $\#\sigma = 1$),  and 
{consequently all of its initial occurrences should be found and stored in $\mathcal{ALLOCC}_1$ through the initialization step
(lines~\ref{pa11}-\ref{pa12}). 
After that, $\mathcal{ALLOCC}_j (j>1)$ is incrementally
computed from $\mathcal{ALLOCC}_{j-1}$ through line (\ref{pa13}).  
}
To determine $\mathcal{ALLOCC}_j$, we avoid computing occurrences leading to redundant right pattern extensions
thanks to the conditions $((sup=\#s) \wedge (\#\sigma > 1))$ in line~(\ref{pa14}). 
Moreover, when computing the right pattern extensions, instead
of storing the part of subsequence $\angx{s[j_1'],\dots,s[j_m']}$, one can
only store the positions of its first and last items $(j_1',j_m')$ in the sequence $sid$.
Consider Example~\ref{example:LR}:
$\REG{\angx{AC}}{\SDB_1}{[0,1]}$ will be encoded as  
$\{(1, \{(4, 5)\})$, $(2, \{(3, 4), (4, 5)$, $(7, 7)\})$, $(3, \{(4, 5)\})$, $(4, \{(4, 4)\})\}$.

Finally, the filtering algorithm handles as efficiently the case {\it without gap constraints}. 
For each pair $(sid, OccSet)$, only the first occurrence $(j_1,j_m)$ in $OccSet$ is determined
thanks to the condition ($N\geq\#s$) in line~(\ref{pa14}).


\subsection{Temporal and spatial complexities of the filtering algorithm}
\label{complexities}

Let $m$$=$$|\SDB|$,
$d$$=$$|\I|$,
and $\ell$ be the length of the longest sequence in $\SDB$.
Computing $\mathcal{ALLOCC}_j$ from $\mathcal{ALLOCC}_{j-1}$ (see function \codex{GetAllOcc} of Algorithm 2)
can be achieved in $O(m \times \ell^2)$. 
The function \codex{getRightExt} (see Algorithm 2) processes all the occurrences of $\sigma$ 
in each sequence of the $\SDB$. 
In the worst case, it may exist $\ell$ occurrences for each sequence in the database. 
So, the time complexity of function \codex{getRightExt} is 
$O(m \times \ell^2+m \times \ell)$ i.e. $O(m \times \ell^2)$. 

\begin{proposition}\label{prop-complexity}
In the worst case, 
(i) filtering can be achieved in $O(m\times\ell^2+d)$
and (ii) the space complexity is $O(m\times \ell^2)$. 
\end{proposition}

\vspace*{-.1cm}
\noindent
{\it Proof: }
(i) The complexity of function \codex{getRightExt} is $O(m \times \ell^2)$.  
The total complexity of function \codex{GetFreqItems} is $O(m \times
\ell)$. Lines~(\ref{filter1}-\ref{filter2}) can be achieved in
$O(d)$. So, the whole complexity is $O(m \times \ell^2 + m \times \ell
+ d)$, i.e. $O(m\times\ell^2+d)$.  

\noindent
(ii) The space complexity of the filtering algorithm lies in the storage of
the $\mathcal{ALLOCC}$ internal data structure.
The occurrences $\mathcal{ALLOCC}_{j}$ of each assignment $\sigma$ in $\SDB$, 
with the length of $\sigma$ varying from $1$ to $\ell$, have to be stored. 
Since it may exist at most $\ell$ occurrences of $\sigma$ in each sequence $sid$,
storing any $\mathcal{ALLOCC}_j$ costs in the worst case $O(m\times \ell)$. Since we can have $\ell$ prefixes,
the worst space complexity of storing all the occurrences $\mathcal{ALLOCC}_j (j=1..\ell)$, is $O(m\times \ell^2)$. 
$\Box$


\vspace*{-.25cm}
\section{Experiments}
\label{section:experimentations}

\setlength{\textfloatsep}{5pt}

\begin{table}[t] \centering
\scalebox{0.8}{
\begin{tabular}{|l|c|c|c|c|c|}
\hline
dataset & $\card\SDB$ & $\card\I$ & avg $(\card s)$ & $\max_{s\in \SDB}$ $(\card s)$ & type of data\\
\hline
Leviathan & 5834 & 9025 & 33.81 & 100 & book\\
\hline
PubMed & 17527 & 19931 &  29 & 198 & bio-medical text\\
\hline
FIFA & 20450 & 2990 & 34.74 & 100 & web click stream \\
\hline
BIBLE & 36369 & 13905 & 21.64 & 100 & bible \\
\hline
Kosarak & 69999 & 21144 & 7.97 & 796   & web click stream\\ 
\hline
Protein & 103120 &  24 &  482 & 600 & protein sequences \\
\hline
\end{tabular}
}
\caption{Dataset Characteristics.} 
\label{table:CharData}
\end{table}

This section reports experiments on several real-life datasets~\cite{JMLR:v15:fournierviger14a,BCCC2012cbms} 
of large size having varied characteristics and representing different application domains (see Tab.~\ref{table:CharData}). 
First, we compare our approach with CP methods
and with the state-of-the-art specialized method \cspade in terms of scalability.
Second, we show the flexibility of our approach for handling different types of constraints simultaneously. 

\begin{figure}[t]
\begin{center}
\vspace*{-.45cm}
\begin{tabular}{c c c}s
FIFA  & & LEVIATHAN  \\
\includegraphics[width=4.6cm, height=3cm]{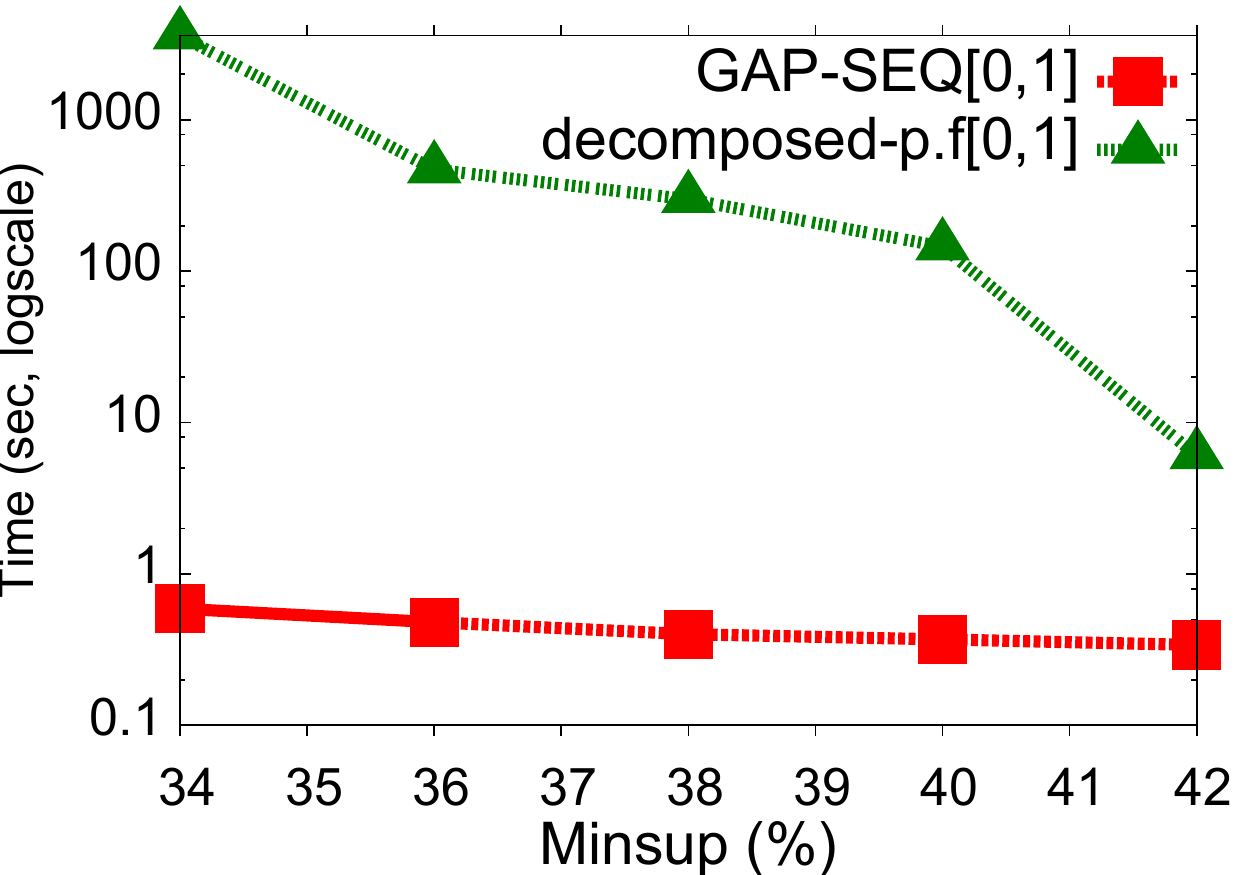}   
&
\hspace{1cm}
&
\includegraphics[width=4.6cm, height=3cm]{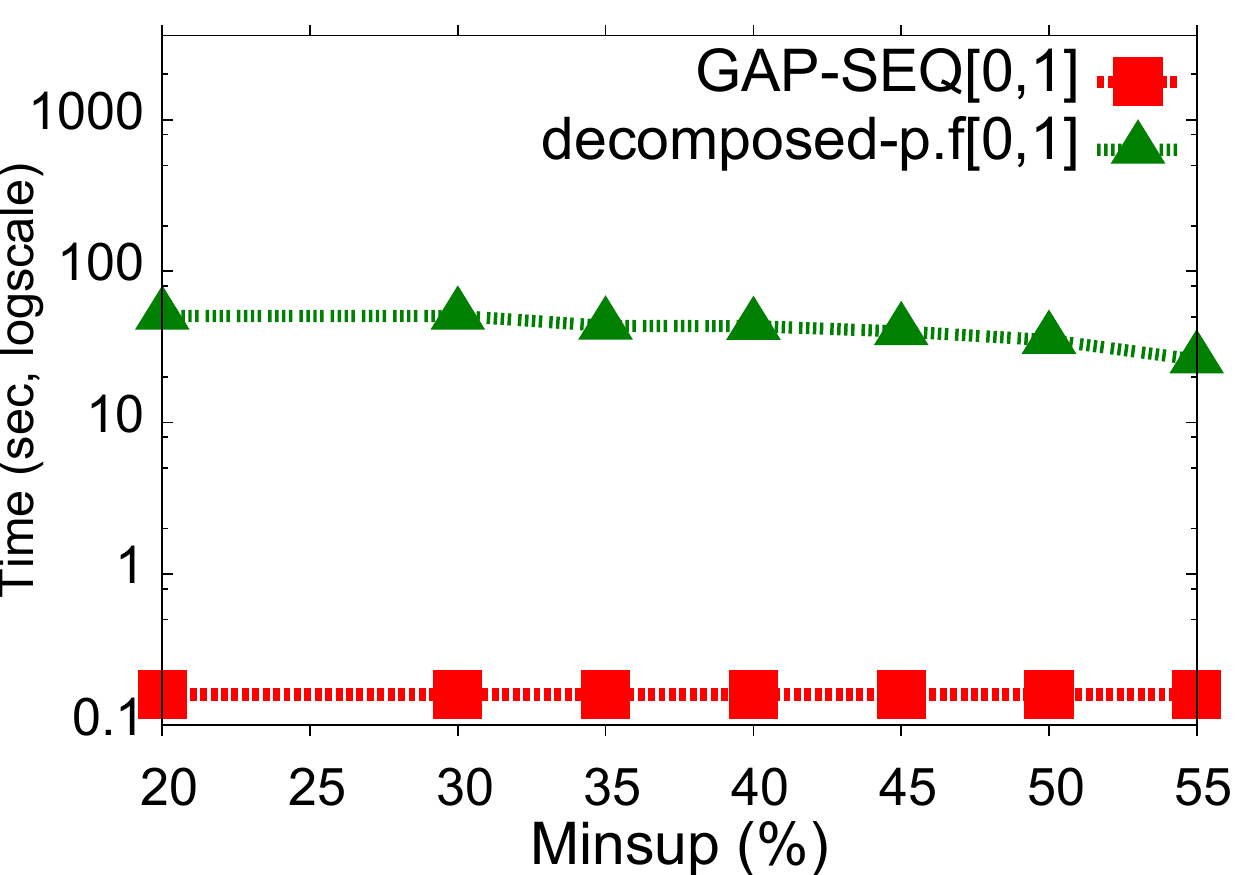}   
\vspace*{-.45cm}
\end{tabular}
\end{center}
\caption{\label{fig:GAP-CPSMD-1-2} Comparing \gapcp with \cps for
  GSPM: CPU times.} 
\end{figure}

\begin{table*}[t] \centering
\scalebox{0.7}{
\begin{tabular}{|l|l|l|r|r|r|r|r|r|}
\hline
\multirow{2}{*}{Dataset} & \multirow{2}{*}{$minsup$ (\%)} &
\multirow{2}{*}{\#PATTERNS} & \multicolumn{2}{c|}{CPU times (s)} & \multicolumn{2}{c|}{\#PROPAGATIONS} & \multicolumn{2}{c|}{\#NODES}\\
\cline{4-9}
&  &  & \gapcp & \cps & \gapcp & \cps & \gapcp & \cps \\
\hline
\multirow{6}{*}{FIFA} & 42 & 1  & 0.34 & 6.06 & 2 & 0  & 1 & 2\\
& 40 & 5 &  0.37 & 144.95 & 10 & 778010  & 6 & 11 \\ 
& 38 & 10 & 0.4 & 298.68 & 20 & 2957965  & 11 & 21\\ 
& 36 & 17 & 0.48 & 469.3  & 34 & 9029578  & 18 & 35\\ 
& 34 & 35 	& 0.59 	&  $-$   & 70 & $-$  & 36 & $-$ \\ 
\hline
\end{tabular}
}
\caption{\gapcp vs. \cps on FIFA dataset.} 
\label{table:GAP-PP:stat}
\end{table*}

\medskip
\noindent
\textbf{Experimental protocol.} 
Our approach was carried out using the {\tt gecode} solver\footnote{\url{http://www.gecode.org}}. 
All experiments were conducted on a processor Intel X5670 with 24 GB of memory.  
A time limit of 1 hour has been set. 
If an approach is not able to complete the extraction within the time limit, it will be reported as ($-$).
$\ell$ was set to the length of the longest sequence of $\SDB$. 
We compare our approach (indicated by \gapcp) with: 
\begin{enumerate}
\item
\cps\footnote{\footnotesize \url{https://dtai.cs.kuleuven.be/CP4IM/cpsm/}}, the most efficient CP methods for GSPM,
\item
\cspade\footnote{\url{http://www.cs.rpi.edu/~zaki/www-new/pmwiki.php/Software/}}, the state-of-the-art specialized method for GSPM,
\item
the \pp global constraint for SPM. 
\end{enumerate}

\smallskip
\noindent
\textbf{(a) GSPM: \gapcp vs the most efficient CP method.}
We compare CPU times for \gapcp and \texttt{deco\-m\-po\-sed-p.f}. 
In the experiments, we used the gap constraint $\gapMN{0}{1}$ and various values of $minsup$.  
Fig.~\ref{fig:GAP-CPSMD-1-2}
shows the results for the two datasets FIFA and LEVIATHAN (results are
similar for other datasets and not reported due to page limitation).
\gapcp clearly outperforms \cps on the two datasets
even for high values of $minsup$: \gapcp is more than an order of
magnitude faster than \cps. For low values of $minsup$, \cps fails to
complete the extraction within the time limit. 

Tab.~\ref{table:GAP-PP:stat} reports for the FIFA dataset and 
different values of $minsup$, the number of calls to the propagate
function of {\tt gecode} (col. $5$) and the number of nodes of the search tree (col. $6$).  
\gapcp is very effective in
terms of number of propagations. For \gapcp, the number of propagations
remains very small compared to \cps (millions). This is due to the
huge number of reified constraints used by \cps to encode the
subsequence relation. Regarding CPU times, \gapcp requires less than
1s. to complete the extraction, while \cps needs much more time to
end the extraction (speed-up value up to $938$).  

\begin{center}
\begin{figure*}[t!]
\vspace*{-.2cm}
\begin{tabular}{ccc}
BIBLE ($0.1\%$) & Kosarak ($0.1\%$) & PubMed ($0.5\%$)   \\
\includegraphics[width=4cm, height=2.9cm]{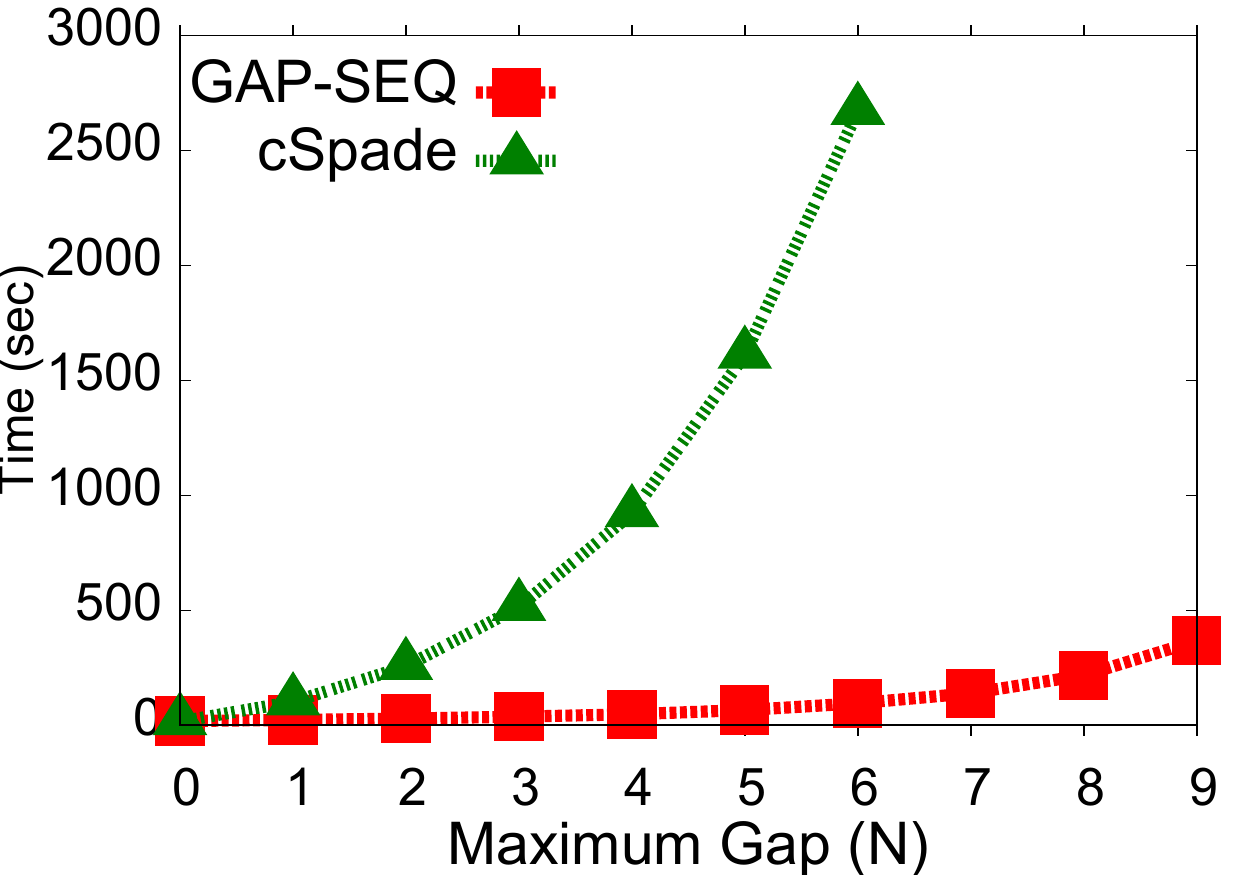}    
&   
\includegraphics[width=4cm, height=2.9cm]{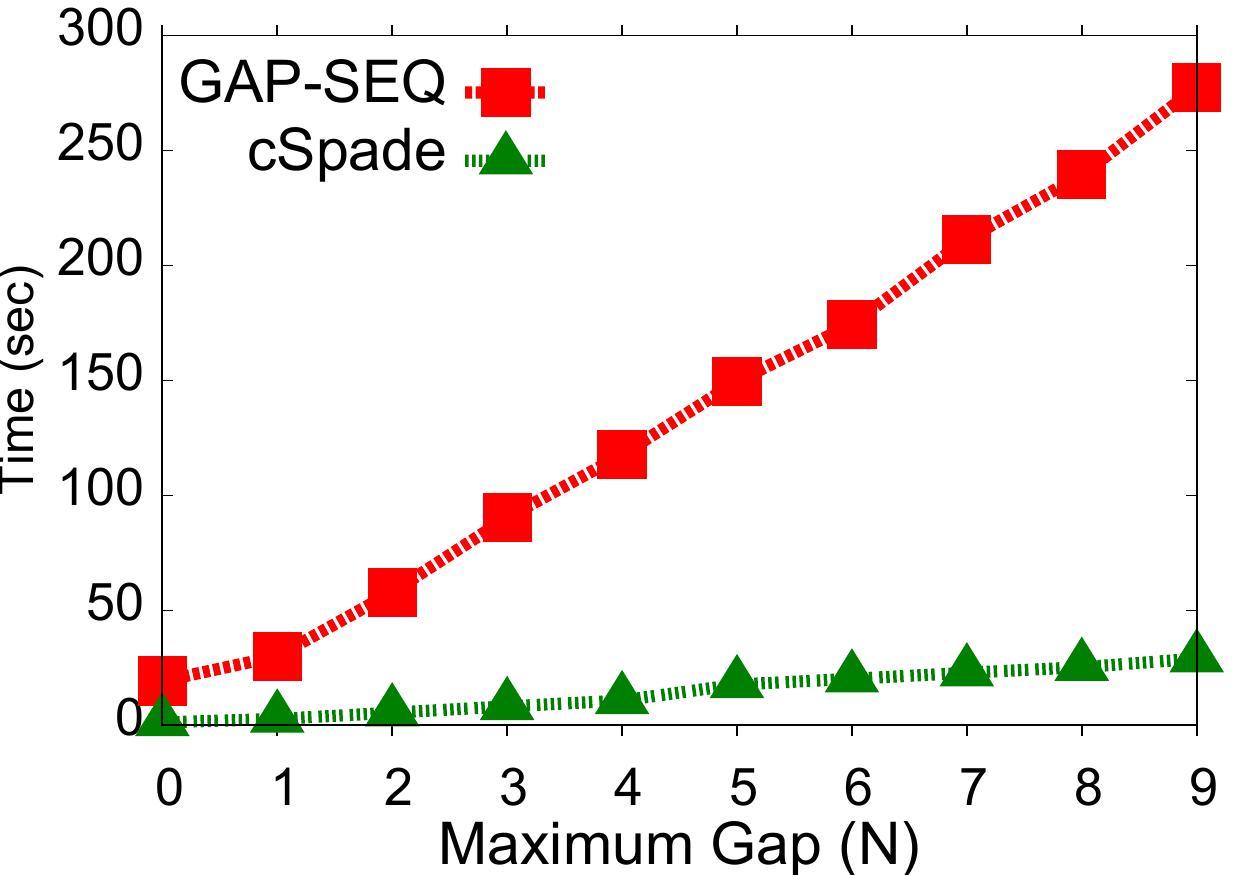}
&
\includegraphics[width=4cm, height=2.9cm]{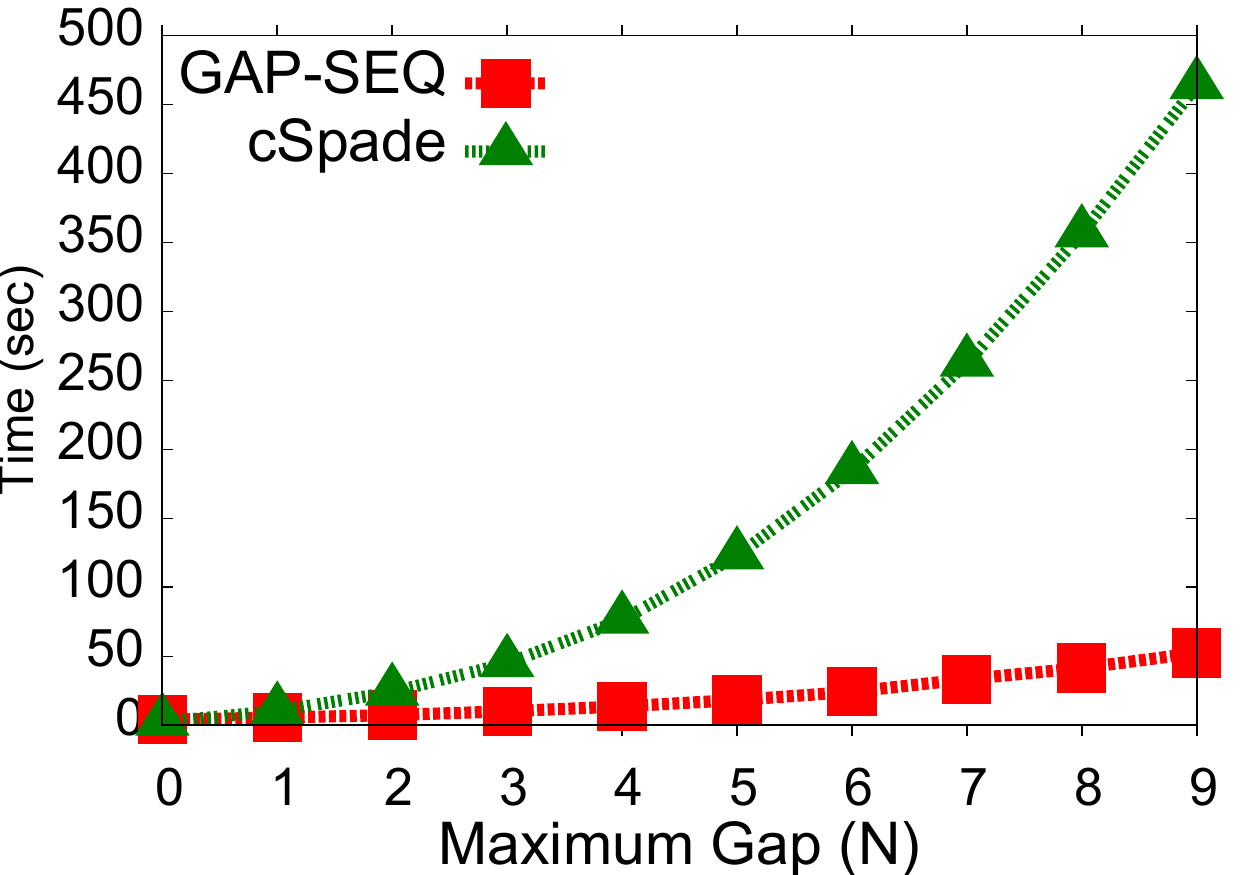}
\\
FIFA ($2\%$) & Leviathan ($0.8\%$) & Protein ($96\%$) \\
\includegraphics[width=4cm, height=2.9cm]{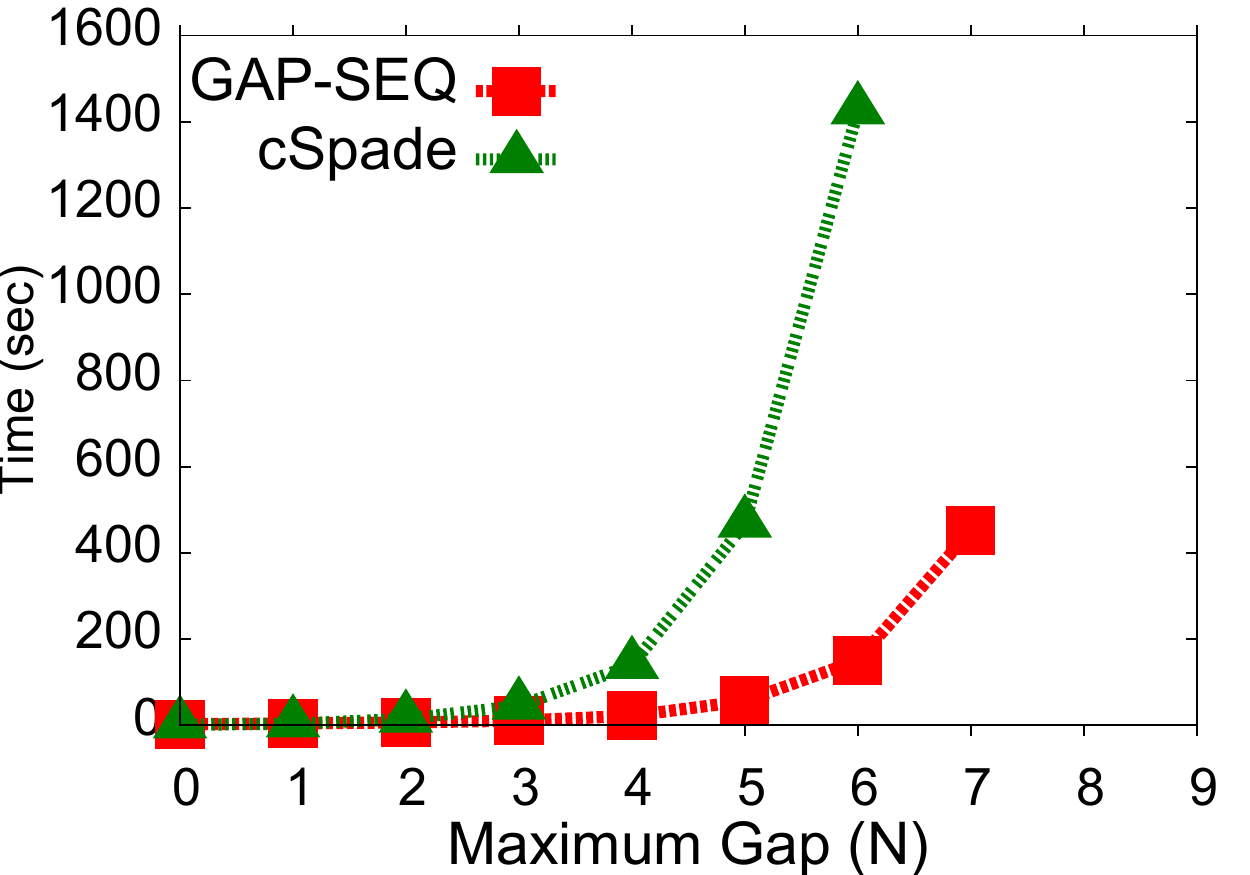}
&
\includegraphics[width=4cm, height=2.9cm]{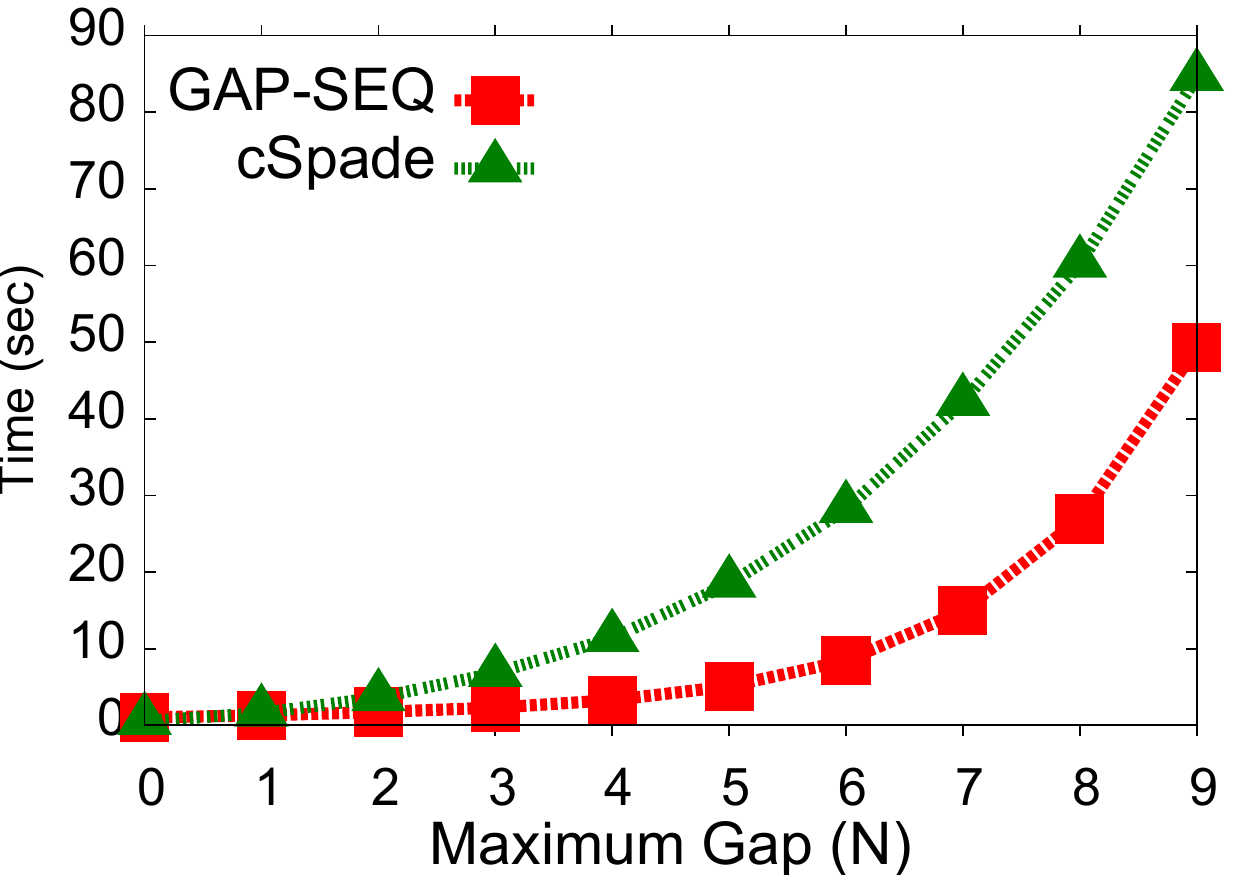}
&
\includegraphics[width=4cm, height=2.9cm]{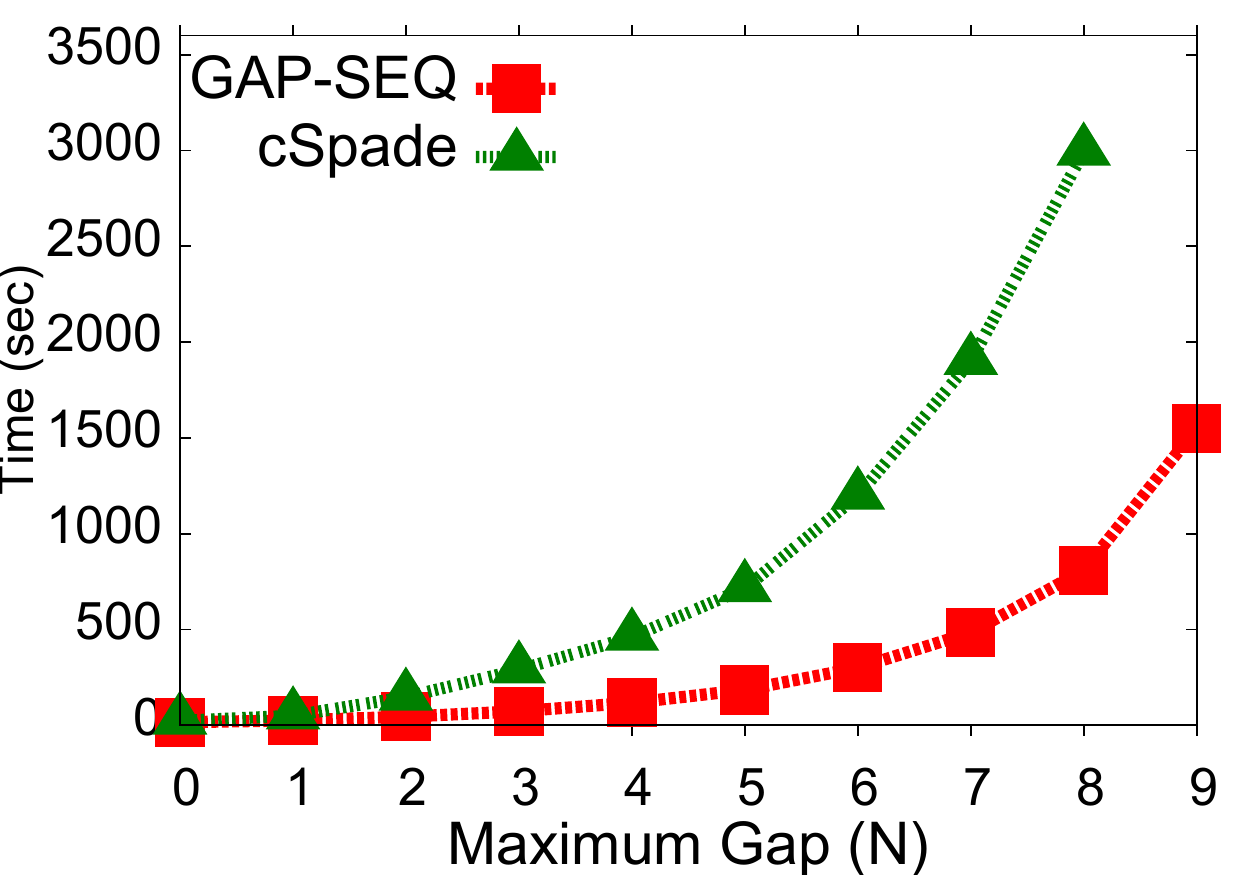}
\vspace*{-.35cm}
\end{tabular}
\caption{\label{fig:GAP-MAX} Varying the value of parameter $N$
  in the gap constraint $(M = 0)$: CPU times.} 
\end{figure*}
\end{center}
\begin{figure*}[t!]
\vspace*{-.2cm}
\begin{tabular}{ccc}
BIBLE  & Kosarak & PubMed   \\
\includegraphics[width=4cm, height=2.9cm]{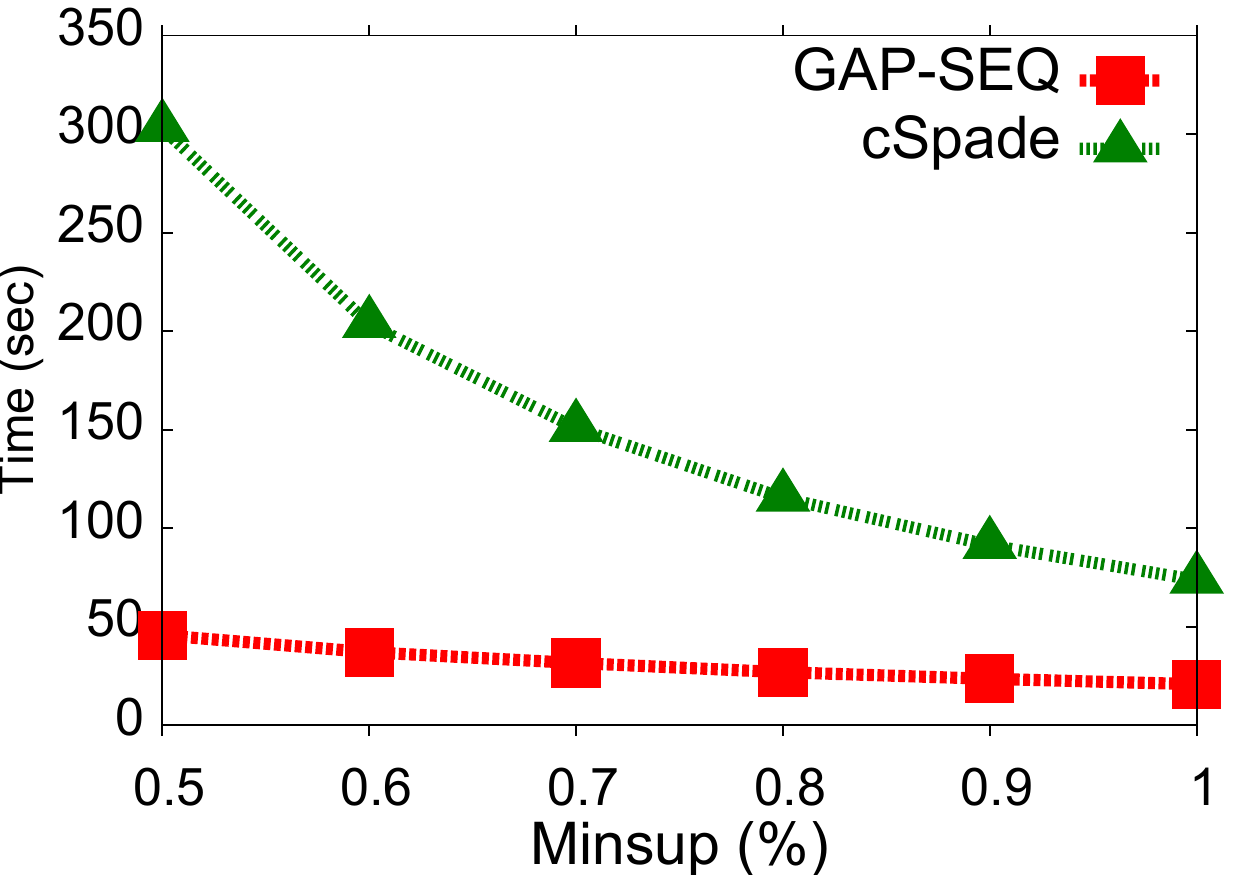}    
&   
\includegraphics[width=4cm, height=2.9cm]{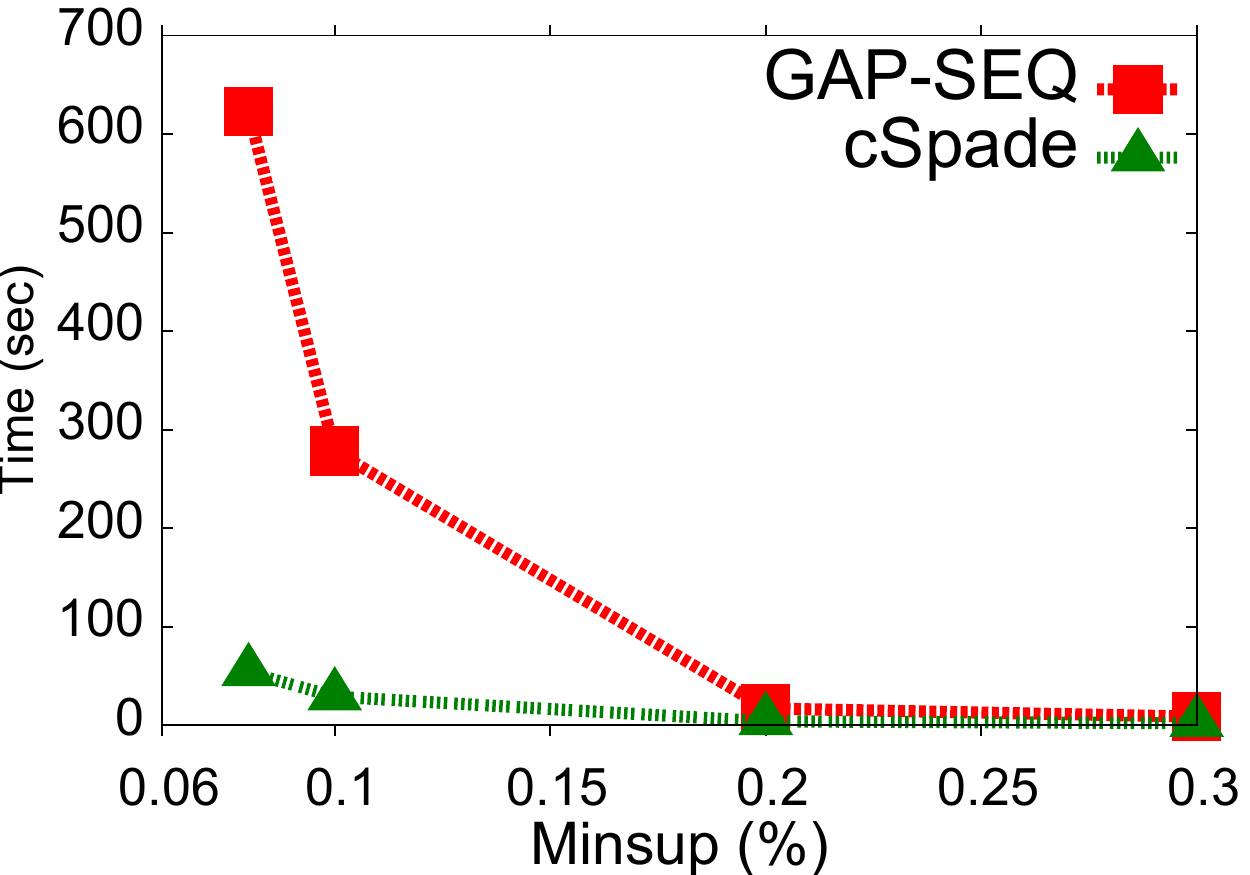}
&
\includegraphics[width=4cm, height=2.9cm]{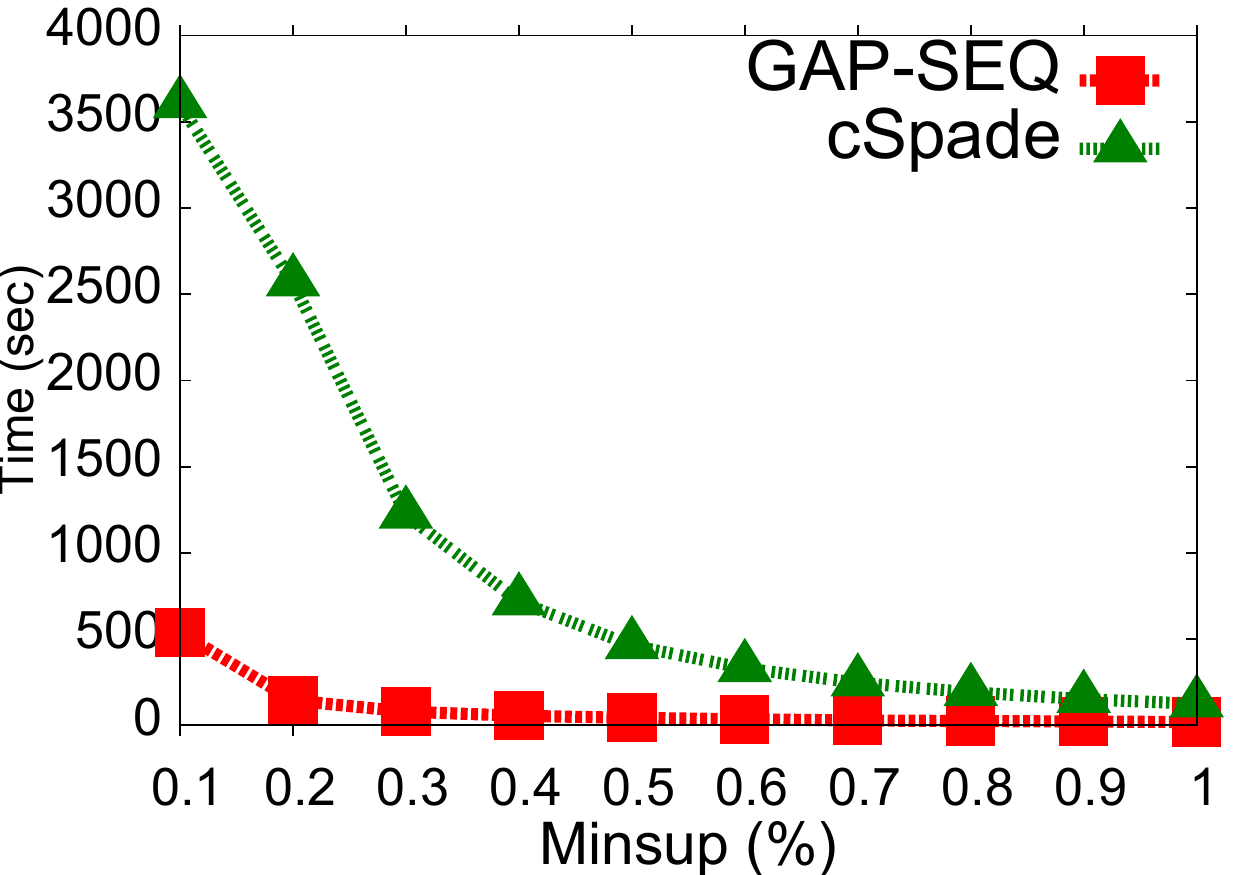}
\\
FIFA  & Leviathan & Protein  \\
\includegraphics[width=4cm, height=2.9cm]{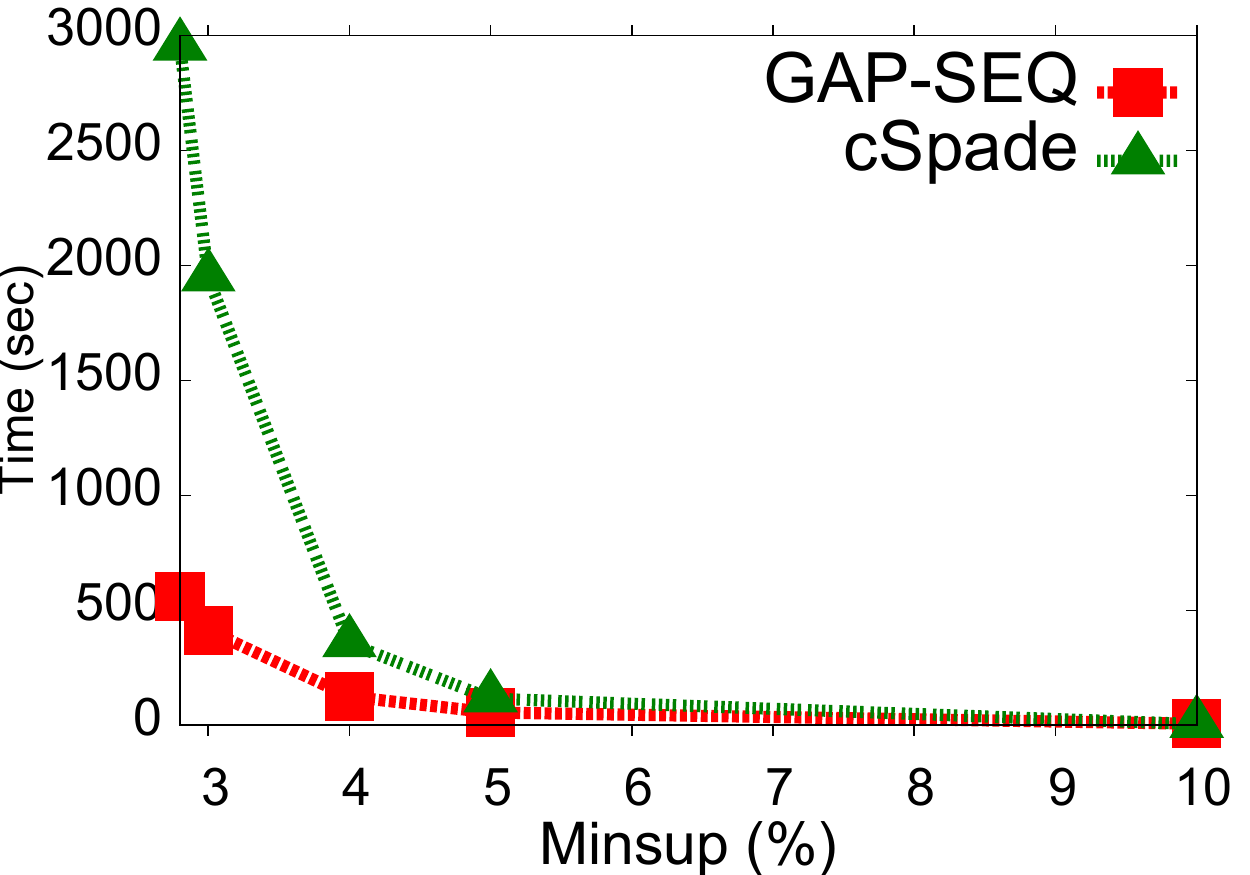}
&
\includegraphics[width=4cm, height=2.9cm]{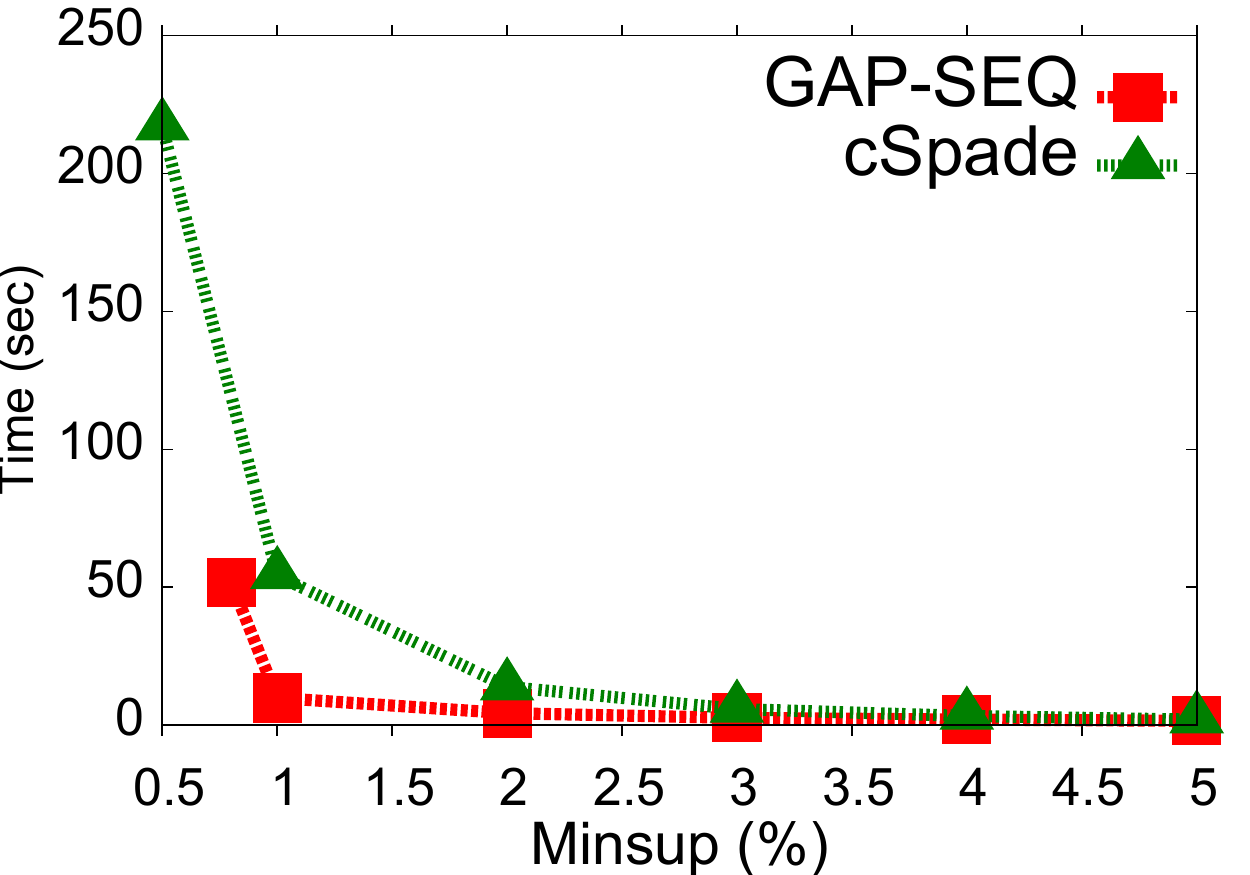}
&
\includegraphics[width=4cm, height=2.9cm]{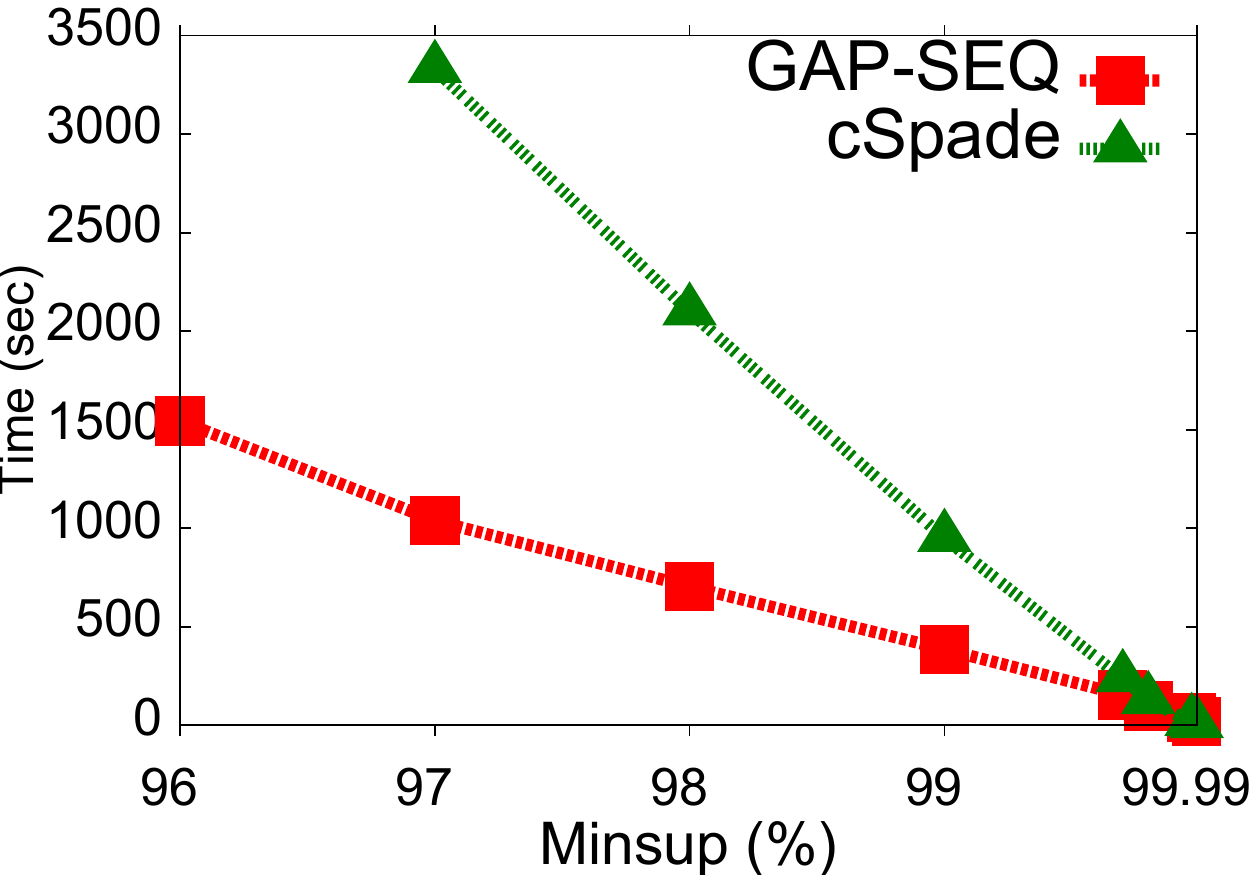}
\vspace*{-.35cm}
\end{tabular}
\caption{\label{fig:GAP-MINSUP} Varying the value of $minsup$
  with the gap constraint $\gapMN{0}{9}$: CPU times.} 
\end{figure*}

\vspace*{-.6cm}
\noindent
\textbf{(b) GSPM: \gapcp vs the state-of-the-art specialized method.}
Second experiments compare \gapcp with \cspade. 
We first fixed $minsup$ to the smallest possible value w.r.t. the
dataset used, and varied the maximum gap $N$ from $0$ to $9$. 
The minimum gap $M$ was set to $0$. 
Fig.~\ref{fig:GAP-MAX} reports the CPU times of both methods. 
First, \gapcp clearly dominates \cspade on all the datasets. 
The gains in terms of CPU times are greatly amplified as the value of $N$ increases.
On FIFA, the speed-up is $9.5$ for $N$$=$$6$. 
On BIBLE, \gapcp is able to complete the extraction for values of $N$
up to $9$ in 433 seconds,  
while \cspade failed to complete the extraction for $N$ greater than
$6$. The only exception is for the Kosarak dataset, where \cspade is
efficient. For this dataset (which is the largest one both in terms
of number of sequences and items), the size of the domains is
important as compared to the other datasets. So, filtering takes much
more time. This probably explains the behavior of \gapcp on this
dataset. 

We also conducted experiments to evaluate how sensitive \gapcp and \cspade are to $minsup$. 
We used the $\gapMN{0}{9}$ constraint, while $minsup$ varied until the two methods
were not able to complete the extraction within the time limit.  
Results are depicted in Fig.~\ref{fig:GAP-MINSUP}. 
Once again, \gapcp obtains the best performance on all datasets (except for Kosarak). 
When the minimum support decreases, CPU times for \gapcp increase reasonably 
while for \cspade they increase dramatically. On PubMed, with $minsup$ set to
$0.1\%$, \cspade finished the extraction after $3,500$ seconds, 
while \gapcp only used $500$ seconds
(speed-up value $7$). These results clearly demonstrate
that our approach is very effective as compared to \cspade on large
datasets. 

\begin{figure*}[t]
\begin{tabular}{ccc}
BIBLE  & Kosarak & Protein   \\
\includegraphics[width=4.08cm, height=3cm]{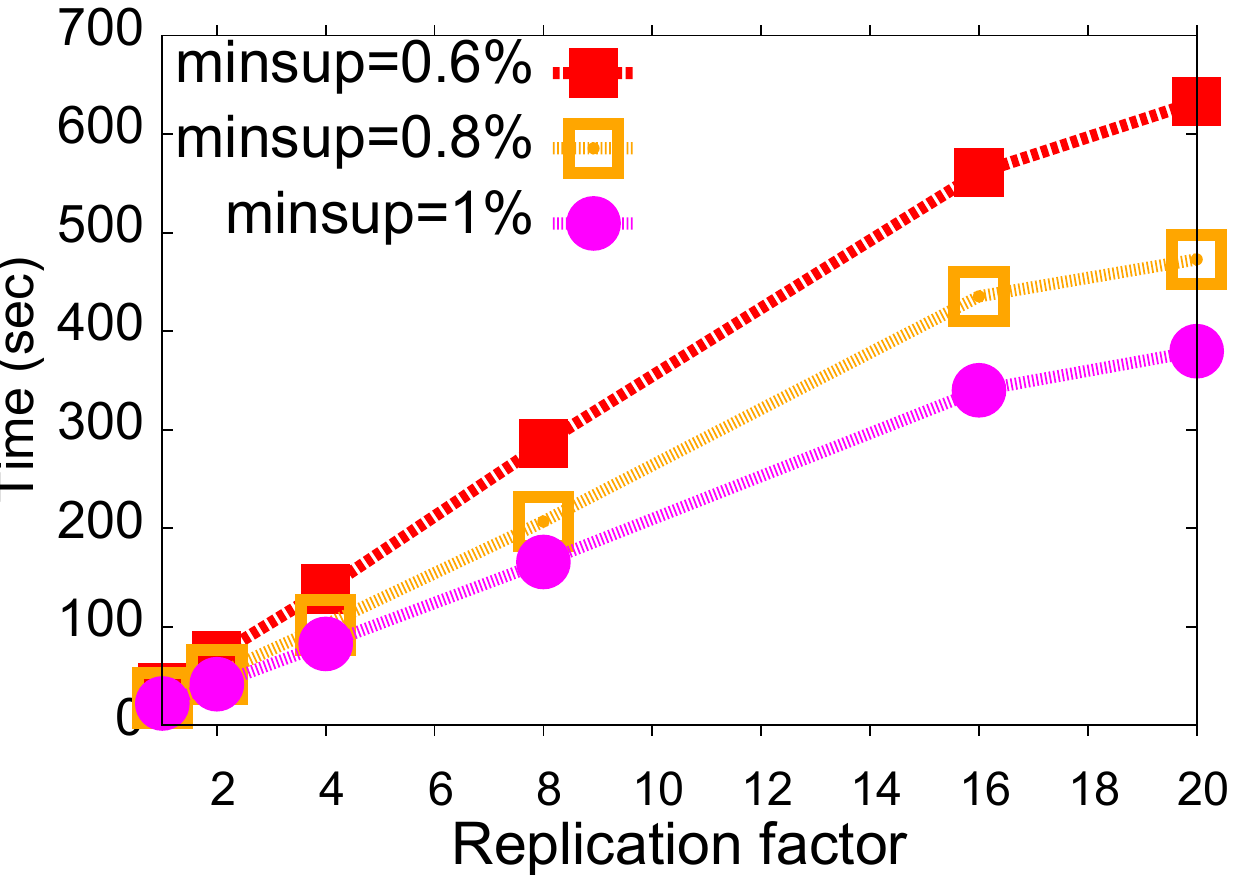}    
&   
\includegraphics[width=4.08cm, height=3cm]{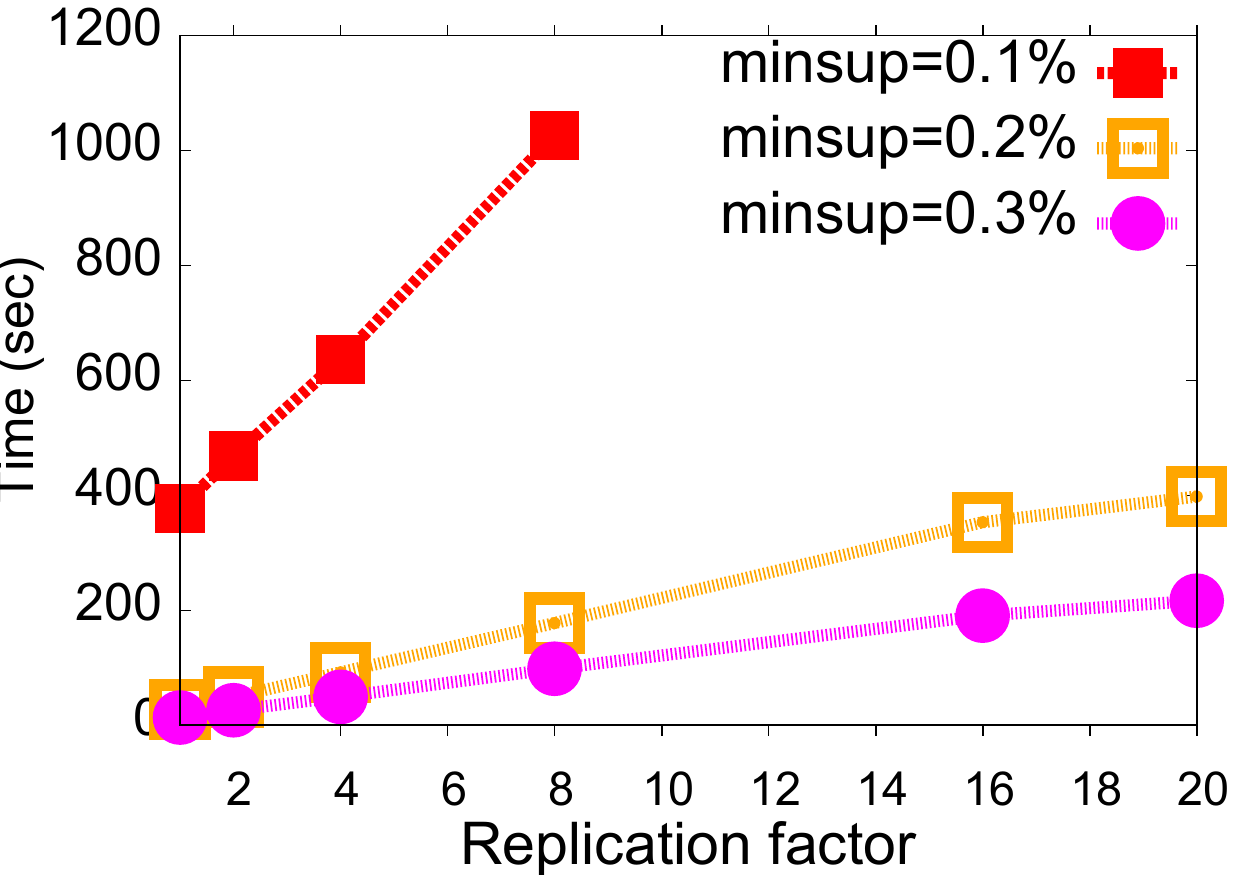}
&
\includegraphics[width=4.08cm, height=3cm]{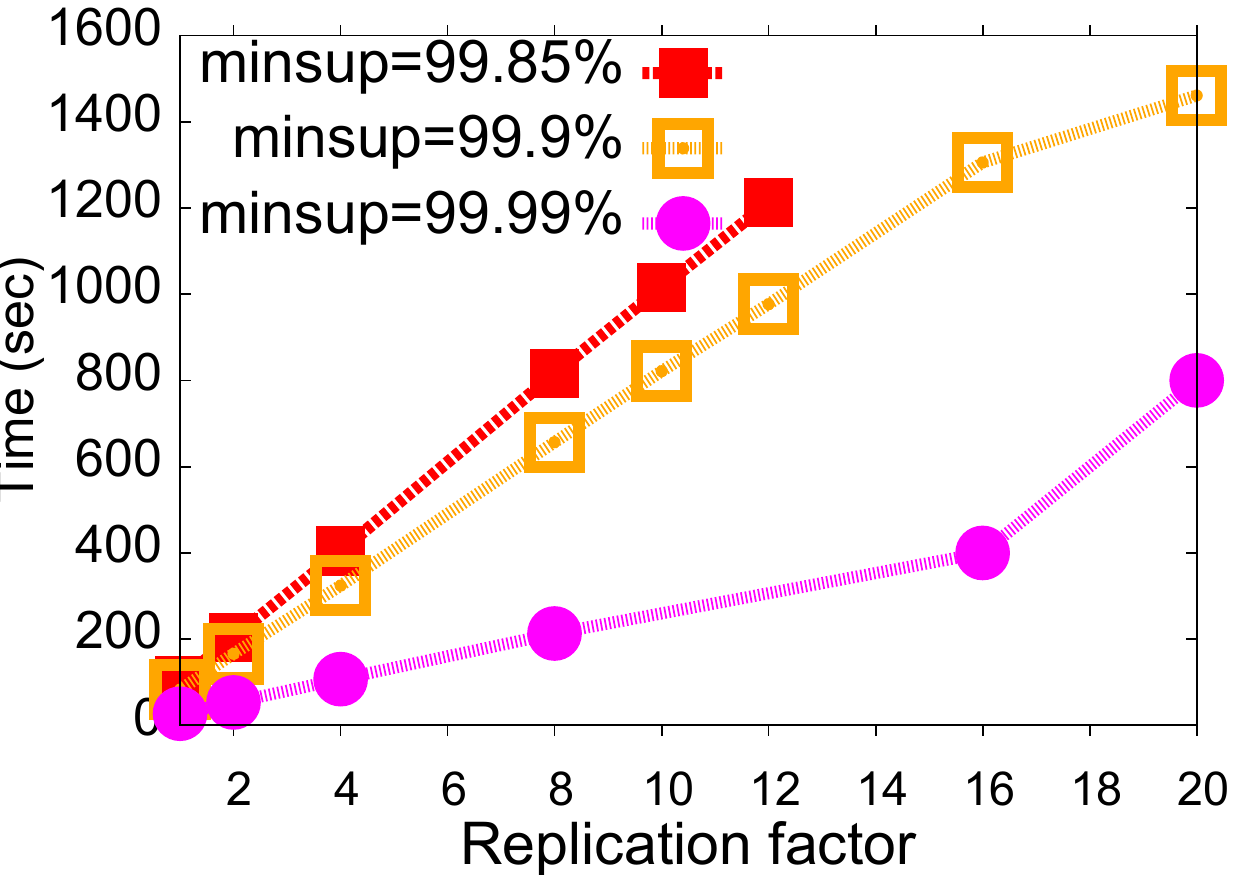}
\vspace*{-.35cm}
\end{tabular}
\caption{\small \label{fig:GAP:scalability} Scalability of \gapcp global
  constraint on BIBLE, Kosarak and Protein.}  
\end{figure*}

\medskip
\noindent
\textbf{(c) GSPM: evaluating the scalability of \gapcp.} 
We used three datasets and replicated them from $1$ to $20$ times.  
The gap constraint was set to $\gapMN{0}{9}$, and $minsup$ to three different values. 
Fig.~\ref{fig:GAP:scalability} reports the CPU times according to the replication factor (i.e. dataset sizes). 
%
CPU times increase (almost) linearly as the number of sequences. 
This indicates that \gapcp achieves scalability while it is a
major issue for CP approaches. The behavior of \gapcp on Protein
is quite different for low values of $minsup$. 
Indeed, for large sequences (such as in Protein), the size of
$\mathcal{ALLOCC}$ may be very 
large and thus checking the gap constraint becomes costly (see
Sect.~\ref{complexities}). 

\medskip
\noindent
\textbf{(d) GSPM: handling various additional constraints.} 
To illustrate the flexibility of our approach, 
we selected the PubMed dataset and stated additional constraints such as 
minimum frequency, minimum size, and other useful constraints
expressing some linguistic knowledge as membership.  
The goal is to extract sequential patterns which convey linguistic
regularities (e.g., gene - rare disease
relationships)~\cite{BCCC2012cbms}.  
The size constraint allows to forbid patterns that are too small
w.r.t. the number of items (number of words) to be relevant patterns; 
we set $\ell_{min}$ to 3.  
The  membership constraint enables to filter out sequential patterns
that do not contain some selected items.  
For example, we state that extracted patterns must contain at least
the two items GENE and DISEASE.  
We used the $\gapMN{0}{9}$ constraint, which is the best setting found 
in~\cite{BCCC2012cbms}.  
As no specialized method exists for this combination of constraints, 
we thus compare \gapcp with and without additional constraints.  

Table~\ref{table:const:stat} reports, for each value of $minsup$,
the number of patterns extracted and the associated CPU times, 
the number of propagations and the number of nodes in the search tree. 
Additional constraints obviously restrict the number of extracted
patterns. 
As the problem is more constrained, the size of the developed search tree is smaller.
Even if the number of propagations is higher, the resulting CPU times are smaller.
To conclude, thanks to the \gapcp global constraint and its encoding,
additional constraints like size, membership and regular expressions
  constraints can be easily stated.

\begin{table*}[t!] \centering
\scalebox{0.9}{
\begin{tabular}{|r|r|r|r|r|r|r|r|r|}
\hline
 \multirow{2}{*}{$minsup$} &
\multicolumn{2}{c|}{\#PATTERNS} & \multicolumn{2}{c|}{CPU times (s)} & \multicolumn{2}{c|}{\#PROPAGATIONS} & \multicolumn{2}{c|}{\#NODES} \\
\cline{2-9}
& gap & gap+size+item & gap & gap+size+item & gap & gap+size+item & gap & gap+size+item\\ 
\hline
1 \% & 14032 & 1805 & 19.34 & 16.83    & 28862 &  47042& 17580 &16584 \\
0.5 \% & 48990 & 6659 &  43.46 &34.6  & 100736 &163205   &  61149 & 58625 \\ 
0.4 \% &  72228 & 10132  & 55.66 & 43.47  & 148597  &240337  & 90477 & 87206 \\ 
0.3 \% &119965 & 17383 & 79.88 & 59.28 &  246934 &398626  & 151280  & 146601 \\ 
0.2 \% & 259760& 39140 	&143.91  	&  100.09   & 534816 &861599   &329185 & 321304 \\ 
0.1 \% & 963053& 153411 	& 539.57 	&  379.04   &1986464  &3186519   &1236340  &1219193 \\ 
\hline
\end{tabular}
}
\caption{\gapcp under size and membership constraints on the PUBMED dataset.} 
\label{table:const:stat}
\end{table*}

\medskip
\noindent
\textbf{(e) Evaluating the ability of \gapCP to efficiently handle SPM.}
In order to simulate the absence of gap constraints, 
we used the ineffective $\gapMN{0}{\ell}$ constraint (recall that $\ell$ is the size of the longest sequence of $\SDB$).
We compared \gapPP{0}{\ell} with 
\pp 
and two configurations of \cspade for SPM: 
\cspade without gap constraint 
and \cspade with $M$ and $N$ set respectively to $0$ and $\ell$, denoted by \cspadeSG{0}{\ell}. 
Let us note that all the above methods will extract the same set of sequential patterns. 

Fig.~\ref{fig:GAP:SPM} reports the CPU times for the four methods. 
First, 
\cspade obtains the best performance (except on Protein). 
These results confirm those observed in~\cite{DBLP:conf/cp/KemmarLLBC15}. 
Second, 
\gapPP{0}{\ell} and \pp exhibit similar behavior, even if
\gapPP{0}{\ell} is slightly less faster. 
So, even if \gapcp handles both cases (with and without gap), it remains very competitive for SPM. 
Third,
\gapPP{0}{\ell} clearly outperforms \cspadeSG{0}{\ell} (except on Kosarak). 
This is probably due to the huge number of unnecessary joining operations performed by \cspadeSG{0}{\ell}.
To conclude, all the performed experiments demonstrate the ability of \gapCP to efficiently handle SPM.

\begin{figure*}[t!]
\begin{tabular}{ccc}
BIBLE  & Kosarak & Protein   \\
\includegraphics[width=4.08cm, height=3cm]{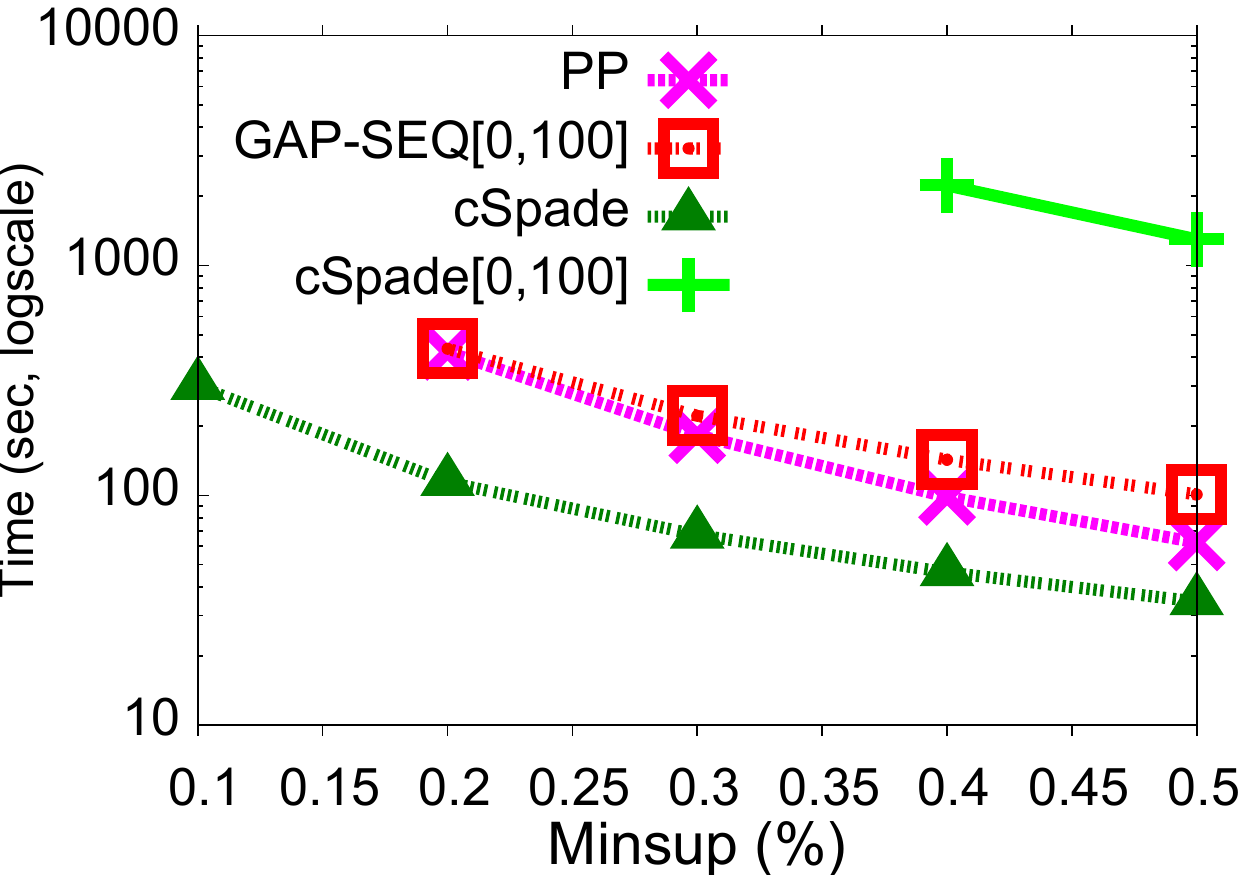}    
&   
\includegraphics[width=4.08cm, height=3cm]{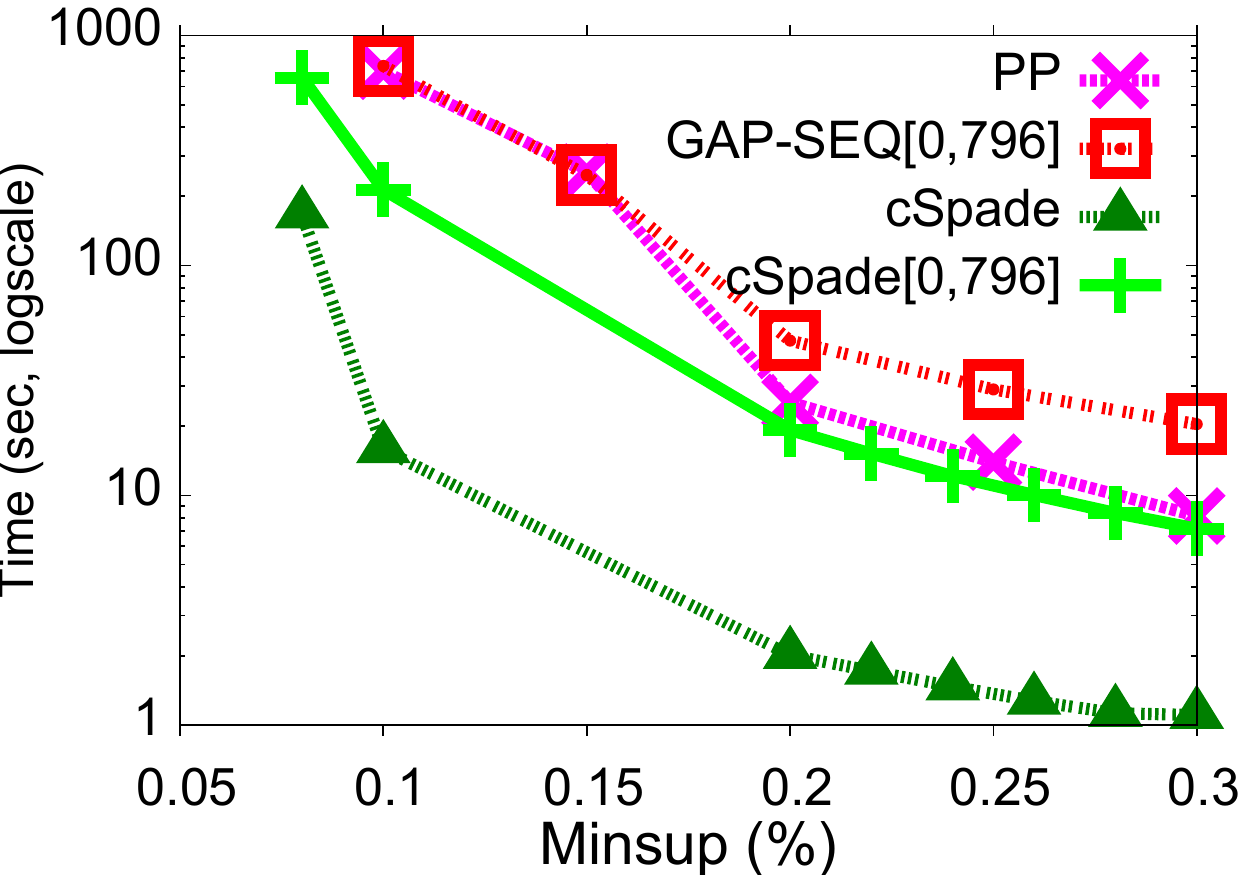}
&
\includegraphics[width=4.08cm, height=3cm]{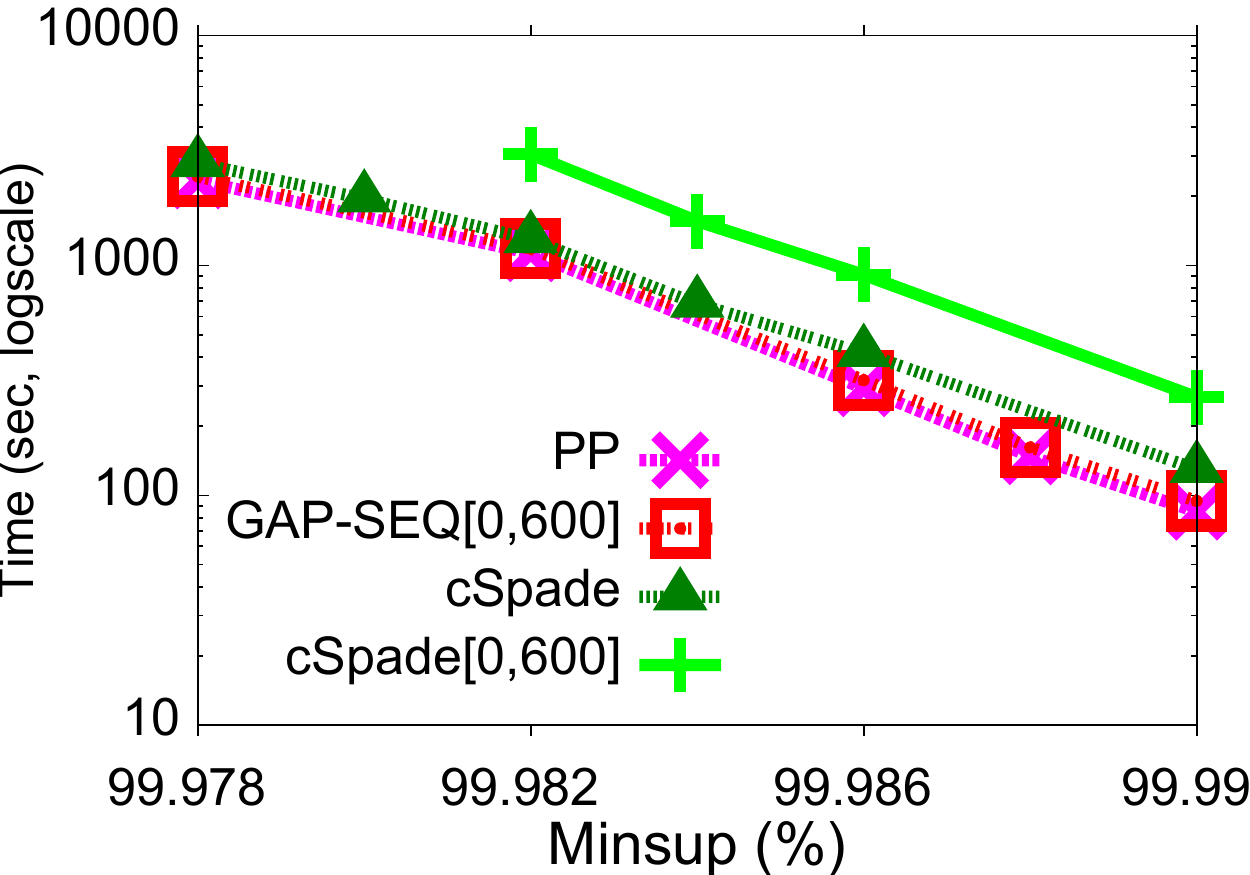}
\vspace*{-.35cm}
\end{tabular}
\caption{\small \label{fig:GAP:SPM} Comparing \gapCP with \pp and
  \cspade for SPM on BIBLE, Kosarak and Protein.}   
\end{figure*}

\medskip

Finally,
the {\tt gecode} implementation of \gapcp and the datasets used in our experiments 
are available online\footnote{\url{https://sites.google.com/site/prefixprojection4cp/}}.

\section{Conclusion}
\label{section:conclusion}
In this paper, we have introduced the global constraint \gapCP enabling to handle SPM with or without gap constraints.
The filtering algorithm takes benefits from the principle of right pattern extensions
and prefix anti-monotonicity property of the gap constraint.
 \gapCP enables to handle several types of constraints simultaneously 
and does not require any reified constraints nor any extra  variables to encode the subsequence relation.
Experiments performed on several real-life datasets 
(i) show that our approach clearly outperforms existing CP approaches as well as specialized methods for GSPM on large datasets,
and
(ii) demonstrate the ability of \gapCP to efficiently handle SPM.

This work opens several issues for future researches. 
We plan to handle constraints on set of sequential
patterns such as closedness, relevant subgroup and skypattern constraints. 

\bibliographystyle{splncs03}
\bibliography{seqcpaior15}

\end{document}